\documentclass[11pt]{article}
\pdfoutput=1
\usepackage[authoryear, round]{natbib}

\usepackage{bm}
\usepackage{fullpage}
\usepackage[english]{babel}
\usepackage[utf8x]{inputenc}
\usepackage[T1]{fontenc}
\usepackage{ae} \usepackage{aecompl}
\usepackage{enumitem}
\setlist{nosep}
\usepackage{microtype}
\usepackage{booktabs}
\usepackage{hhline}
\usepackage{makecell}

\usepackage{etoolbox}
\apptocmd{\sloppy}{\hbadness 10000\relax}{}{}

\usepackage{amsmath,amsthm,amssymb,mathtools,thmtools}
\usepackage{graphicx}
\usepackage[textsize=scriptsize,disable]{todonotes}
\usepackage[colorlinks=true, allcolors=blue]{hyperref}
\usepackage{algorithm,algorithmicx,algpseudocode}
\usepackage{xfrac}
\usepackage[normalem]{ulem}

\hypersetup{colorlinks,
            breaklinks=true,
            linkcolor=blue,
            citecolor=blue,
            urlcolor=magenta,
            linktocpage,
            plainpages=false,
            bookmarks=false}

\usepackage{notations_arxiv}
\usepackage{comment}
\usepackage{tikz,pgfplots,pgfplotstable}
\usepgfplotslibrary{units}
\usepackage{tabularx} 
\usepackage{subfig}
\usepackage{array}
\usepgfplotslibrary{groupplots}

\title{On the Role of Noise in the Sample Complexity of Learning Recurrent Neural Networks: Exponential Gaps for Long Sequences}

\author{
    Alireza Fathollah Pour\thanks{McMaster University, \texttt{fathola@mcmaster.ca}}
    \and 
    Hassan Ashtiani\thanks{McMaster University, \texttt{zokaeiam@mcmaster.ca}. Hassan Ashtiani is also a faculty affiliate at Vector Institute and supported by an NSERC Discovery Grant.}
}

\numberwithin{equation}{section}

\begin{document}

\maketitle

\begin{abstract}

  We consider the class of noisy multi-layered sigmoid recurrent neural networks with $w$ (unbounded) weights for classification of sequences of length $T$, where independent noise distributed according to $\cN(0,\sigma^2)$ is added to the output of each neuron in the network. Our main result shows that the sample complexity of PAC learning this class can be bounded by $O (w\log(T/\sigma))$. For the non-noisy version of the same class (i.e., $\sigma=0$), we prove a lower bound of $\Omega (wT)$ for the sample complexity. 
  Our results indicate an exponential gap in the dependence of sample complexity on $T$ for noisy versus non-noisy networks. Moreover, given the mild logarithmic dependence of the upper bound on $1/\sigma$, this gap still holds even for numerically negligible values of $\sigma$.
  
\end{abstract}

\section{Introduction}

Recurrent Neural Networks (RNNs) are effective tools for processing sequential data. They are used in numerous applications such as speech recognition~\citep{graves2013speech}, computer vision~\citep{karpathy2015deep}, translation~\citep{sutskever2014sequence}, modeling dynamical systems~\citep{hardt2016gradient} and time series~\citep{qin2017dual}. Recurrent models allow us to design classes of predictors that can be applied to (i.e., take input values from) sequences of arbitrary length. For processing a sequence of $T$ elements, a predictor $f$ (e.g., a neural network) ``consumes'' the input elements one by one, generating an output at each step. This output is then used in the next step (as another input to $f$ along with the next element in the input sequence). Defining recurrent models formally takes some effort, and we relegate it to the next sections. In short, the function $f$ is (recursively) applied $T$ times in order to generate the ultimate outcome.

Let us fix a base class $\cF_w$ of all multi-layered feed-forward sigmoid neural networks with $w$ weights. We can create a recurrent version of this class, which we will denote by $\brmodel{\cF_w}{T}$, for classifying sequences of length $T$. One can study the sample complexity of PAC learning $\brmodel{\cF_w}{T}$ with respect to different loss functions. \citet{koiran1998vapnik} studied the binary-valued version of this class by applying a threshold function at the end, and proved a lower bound of $\Omega(wT)$ for its VC dimension. 

There has also been efforts for proving upper bounds on the sample complexity of PAC learning $\brmodel{\cF}{T}$ for various base classes $\cF$ and different loss functions. Given the above lower bound, a gold standard has been achieving a linear dependence on $T$ in the upper bound. \citet{koiran1998vapnik} proved an upper bound of $O(w^4 T^2)$ on the VC dimension of $\brmodel{\cF_w}{T}$ discussed above. More recent papers have considered the more realistic setting of classification with continuous-valued RNNs, e.g., by removing the threshold function and using a bounded Lipschitz surrogate loss. In this setting, \citet{zhang2018stabilizing} proved an upper bound of $\widetilde{O}(T^4 w\|W\|^{O(T)})$ on the sample complexity\footnote{Ignoring the dependence of the sample complexity on the accuracy and confidence parameters.} where $\|W\|$ is the spectral norm of the network. \citet{chen2020generalization} improved over this result by proving an upper bound of 
$\widetilde{O}(Tw\|W\|^2 \min \{\sqrt{w}, \|W\|^{O(T)} \})$. These bounds get close to the gold standard when the spectral norm of the network satisfies $\|W\|\leq 1$.

The above upper bounds are proved by simply ``unfolding'' the recurrence,  effectively substituting the recurrent class $\brmodel{\cF_w}{T}$ with the (larger) class of $T$-fold compositions $\cF_w\circ \cF_w \ldots \circ \cF_w$. These unfolding techniques do not exploit the fact that the function $f$ (that is applied recursively for $T$ steps to compute the output of the network) is fixed across all the $T$ steps. Consequently, the resulting sample complexity has (super-)linear dependence on $T$. Therefore, we would need a prohibitively large sample size for training recurrent models for classifying very long sequences. Nevertheless, this dependence is inevitable in light of the of lower bound of~\cite{koiran1998vapnik}. Or is it?

In this paper, we consider a related class of \emph{noisy} recurrent neural networks, $\brmodel{\widetilde{\cF_w^\sigma}}{T}$. The hypotheses in this class are similar to those in $\brmodel{\cF_w}{T}$, except that outputs of (sigmoid) activation functions are added with independent Gaussian random variables, $\cN(0,\sigma ^2)$.
Our main result demonstrates that, remarkably, the noisy class can be learned with a number of samples that is only logarithmic with respect to $T$.

\begin{theorem}[Informal version of Theorem~\ref{thm:upper_rnn}]
    The sample complexity of PAC learning the class $\brmodel{\widetilde{\cF_w^\sigma}}{T}$ of noisy recurrent networks with respect to ramp loss is $\widetilde{O}(w\log (T/\sigma))$.
\end{theorem}

One challenge of proving the above theorem is that the analysis involves dealing with \emph{random} hypotheses. Therefore, unlike the usual arguments that bound the covering number of a set of deterministic maps with respect to the $\ell_2$ distance, we study the covering number of a class of random maps with respect to the total variation distance. We then invoke some of the recently developed tools in \cite{pour2022benefits} for bounding these covering numbers. Another challenge is deviating from the usual ``unfolding method'' and exploiting the fact that in recurrent models a \emph{fixed} function/network is applied recursively. 

The mere fact that learning $\brmodel{\widetilde{\cF_w^\sigma}}{T}$ requires less samples compared to its non-noisy counterpart is not entirely unexpected. For classification of long sequences, however, the sample complexity gap is quite drastic (i.e., exponential). We argue that a logarithmic dependency on $T$ is actually more realistic in practical situations: for finite precision machines, one can effectively break the $\Omega(T)$ barrier even for non-noisy networks.
To see this, let us choose $\sigma$ to be a numerically negligible number (e.g., smaller than the numerical precision of our computing device). In this case, the class of noisy and non-noisy networks become effectively the same when implemented on a device with finite numerical precision. But then our upper bound shows a mild logarithmic dependence on $1/\sigma$.

One caveat in the above argument is that the lower bound of~\citet{koiran1998vapnik} is proved for the 0-1 loss and perhaps not directly comparable to the setting of the upper bound which uses a Lipcshitz surrogate loss. We address this by showing a comparable lower bound in the same setting.

\begin{theorem}[Informal version of ~Theorem~\ref{thm:lower_rnn}]
The sample complexity of PAC learning $\brmodel{\cF_w}{T}$ with ramp loss is $\Omega\left(wT\right)$.
\end{theorem}
In the next section we introduce our notations and define the PAC learning problem in Section~\ref{sec:prelim}. We state the lower bound in Section~\ref{sec:lower_bound}, and the upper bound in Section~\ref{subsec:upper_bound}. Sections~\ref{sec:covering_classic}, \ref{sec:covering_random}, and \ref{sec:proof_scheme} provide a high-level proof of our upper bound. We conclude our results and discuss future works in Section~\ref{sec:conclusion}. Finally, we discuss additional related work in Section~\ref{app:related}.

\section{Preliminaries}\label{sec:prelim}

\subsection{Notations}

$\|x\|_1,\|x\|_2$, and $\|x\|_{\infty}$ denote the $\ell_1,\ell_2$, and $\ell_{\infty}$ norms of a vector $x\in \bR^d$ respectively. We denote the cardinality of a set $S$ by $|S|$. The set of natural numbers smaller or equal to $m$ is represented by $[m]$. A vector of all zeros is denoted by $\zero{d} =\begin{bmatrix} 0 \, \ldots \, 0\end{bmatrix} \transpose \in \bR^{d}$.
We use $\cX\subseteq \bR^d$ as a domain set. We will study classes of vector-valued functions; a hypothesis is a Borel function $f:\bR^d\rightarrow\bR^p$, and a hypothesis class $\cF$ is a set of such hypotheses.

We find it useful to have an explicit notation---here an overline---for the random versions of the above definitions: $\rv{\cX}$ is the set of all random variables defined over ${\cX}$ that admit a generalized density function\footnote{Both discrete (by using Dirac delta function) and absolutely continuous random variables admit a generalized density function.}. $\rv{x}\in \rv{\cX}$ is a random variable in this set. To simplify this notation, we sometimes just write $\rv{x}\in \bR^d$ rather than  $\rv{x}\in \rv{\bR^d}$.

$\rv{y}=f(\rv{x})$ is the random variable associated with pushforward of $\rv{x}$ under Borel map ${f}:\bR^d\rightarrow\bR^p$. We use $\rv{f}:\bR^d\rightarrow\bR^p$ to indicate that the mapping itself is random. 
Random hypotheses can be applied to both random and non-random inputs---e.g., $\rv{f}(\rv{x})$ and $\rv{f}(x)$\footnote{Technically, we consider $\rv{f}(x)$ to be $\rv{f}(\rv{\delta_x})$, where $\rv{\delta_x}$ is a random variable with Dirac delta measure on $x$.}.
A class of random hypotheses is denoted by $\rv{\cF}$ .
\begin{definition}[Composition of Two Hypothesis Classes]\label{def:composition}
We denote by $h\circ f$ the function $h(f(x))$ (assuming the range of $f$ and the domain of $h$ are compatible).
The composition of two hypothesis classes $\cF$ and $\cH$ is defined by $\cH\circ\cF=\{h\circ f\ \mid h\in \cH, f\in \cF\}$. Composition of classes of random hypotheses is defined similarly by $\rv{\cH}\circ\rv{\cF}=\{\rv{h}\circ \rv{f}\ \mid \rv{h}\in \rv{\cH}, \rv{f}\in \rv{\cF}\}$.
\end{definition}

\subsection{Feedforward neural networks}

We will first define some classes associated with feedforward neural networks. Let $\phi(x)=\frac{1}{1+e^{-x}}-\frac{1}{2}$ be the centered sigmoid function. $\Phi:\bR^p\rightarrow {[-1/2,1/2]}^p$ is the element-wise sigmoid activation function defined by $\Phi((x^{(1)},\ldots, x^{(p)}))=(\phi(x^{(1)}),\ldots, \phi(x^{(p)}))$. 

\begin{definition}[Single-Layer Sigmoid Neural Networks]\label{def:neuralnet}
The class of single-layer sigmoid neural networks with $d$ inputs and $p$ outputs is defined by $\net{d}{p}=\{f_W:\bR^d\to[-1/2,1/2]^p \mid f_W(x)=\Phi(W\transpose x), W\in\bR^{d\times p}\}$.  
\end{definition}

Based on Definition~\ref{def:neuralnet}, we can define the class of multi-layer (feedforward) neural networks (with $w$ weights) as a composition of several single-layer networks. Note that the number of hidden neurons can be arbitrary as long as the total number of weights/parameters is $w$. 
\begin{definition}[Multi-Layer Sigmoid Neural Networks]\label{def:multinet}
A class of multi-layer sigmoid networks with 
$p_0$ inputs, $p_k$ outputs, and $w$ weights that take inputs in $[-1/2,1/2]^{p_0}$ is defined by 
\begin{equation*}
   \mnet{p_0}{p_k}{w} = \bigcup\net{p_{k-1}}{p_k} \circ \ldots \circ \net{p_0}{p_1}
\end{equation*}
where union is taken over all choices of $(p_1, p_2, \ldots, p_{k-1})\in \bN^{k-1}$ that satisfy $\sum_{i=1}^{k} p_i.p_{i-1}=w$. We say $\mnet{p_0}{p_k}{w}$ is well-defined if the union is not empty.
\end{definition}
Well-definedness basically means that $p_0, p_k$, and $w$ are compatible. For simplicity, in the above definition we restricted the input domain to $[-1/2,1/2]^d$. This will help in defining the recurrent versions of these networks (since the input and output domains become compatible). However, our analysis can be easily extended to capture any bounded domain (e.g., $[-B,B]^{d}$).

\subsection{Recursive application of a function and recurrent models}

In this section we define $\brmodel{\cF}{T}$ which is the recurrent version of class $\cF$ for sequences of length $T$. Let $v=\left( a_1,\ldots,a_m\right)\in\cX^m$ for $m\in\bN$. We define $\first{v}=\left( a_{1}, \ldots, a_{m-1}\right)\in\cX^{m-1}$ and $\last{v}=a_m \in \cX$ as functions that return the first $m-1$ and the last dimensions of the vector $v$, respectively. Let $u^{(0)}, u^{(1)}, \ldots, u^{(T-1)}$ be a sequence of inputs, where $u^{(i)}\in \bR^p$, and let $f:\bR^s\to \bR^q$ be a hypothesis/mapping. In the context of recurrent models, it is useful to define the recurrent application of $f$ on this sequence. Note that out of the $q$ dimensions of the range of $f$, $q-1$ of them are recurrent and therefore are fed back to the model. Basically, $f^R\left( U,t \right)$ will be the result of applying $f$ on the first $t$ elements of $U$ (with recurrent feedback). 

\begin{definition}[Recurrent Application of a Function] \label{def:rec_application}
Let $U = \begin{bmatrix}u^{(0)} \ldots u^{(i)} \ldots u^{(T-1)}\, \end{bmatrix}\in\bR^{p\times T}$ be a sequence of inputs of length $T$, where $u^{(i)}\in\bR^{p}$ denotes the $i$-th column of $U$ for $0 \leq i\leq T-1$. Let $f$ be a (random) function from $\bR^s$ to $\bR^q$, where $s = p+q-1$. Moreover, define $f^R\left( U,0 \right) = f\left(  \begin{bmatrix} \zero{q-1} & u^{(0)}\end{bmatrix}\transpose \right)$.
Then, for any $1 \leq t \leq T-1$,  the recursive application of $f$ is denoted by $f^R:\bR^{p\times T} \times [T-1]\rightarrow \bR^q$ and is defined as $f^R\left( U,t \right) = f\left( \begin{bmatrix}  \first{f^R\left( U,t-1 \right)} & u^{(t)}\end{bmatrix}\transpose \right)$.
\end{definition}
Now we are ready to define the (recurrent) hypothesis class $\brmodel{\cF}{T}$. Each hypothesis in this class takes a sequence $U$ of input vectors, and applies a function $f\in\cF$ recurrently on the elements of this sequence. The final output will be a real number. We give the formal definition in the following; also see Figure~\ref{fig:rec_model} for a visualization. 
\begin{definition}[Recurrent Class]\label{def:rnet}
Let $s,p,q\in \bN$ such that $s=p+q-1$.
Let $\cF$ be a class of functions from $\bR^s$ to $\bR^q$. The class of recurrent models with length $T$ that use functions in $\cF$ (which we denote by recurring class) as their recurring block is defined by
$$\brmodel{\cF}{T}=\{h:\bR^{p \times T} \to \bR \mid h(U) = \last{f^R\left( U,T-1 \right)},f\in\cF\}$$ 
\end{definition}
For example, $\brmodel{\mnet{p_0}{p_k}{w}}{T}$ is the class
of (real-valued) recurrent neural networks with length $T$ that use $\mnet{p_0}{p_k}{w}$ as their recurring block.
We say that $\brmodel{\mnet{p_0}{p_k}{w}}{T}$ is \emph{well-defined} if $\mnet{p_0}{p_k}{w}$ is well-defined and also the input/output dimensions are compatible (i.e., $p_0\geq p_k$).

\begin{figure}
\centering    {\usetikzlibrary{fit}
\usetikzlibrary{decorations.pathreplacing}

\begin{tikzpicture}[scale=0.9,>=stealth]\label{Noisy}

\def\xa{0}
\def\xb{6}
\def\xc{12}

\def\za{\footnotesize $\zero{q-1}$}
\def\zb{\tiny  $\first{f^{R}(U,t-1)}$}
\def\zc{\tiny  $\first{f^{R}(U,T-2)}$}

\def\oa{\tiny $f^{R}(U,0)$}
\def\ob{\tiny $f^{R}(U,t)$}
\def\oc{\tiny $f^{R}(U,T-1)$}

\def\foa{\tiny  $\first{f^{R}(U,0)}$}
\def\fob{\tiny  $\first{f^{R}(U,t)}$}
\def\foc{\tiny  $\first{f^{R}(U,T-1)}$}

\def\ia{\footnotesize $u^{(0)}$}
\def\ib{\footnotesize $u^{(t)}$}
\def\ic{\footnotesize $u^{(T-1)}$}

\foreach \x in {\xa,\xb,\xc}{
\def\a{\ifthenelse{\x=\xa}{\za}{\ifthenelse{\x=\xb}{\zb}{\zc}}}
	\node[rectangle,draw, fill = teal!20!white, minimum width = 0.75cm, minimum height=2.3cm,align=center] at  (\x+1.175,1.7) (lay1) {\footnotesize $f$};

    \node[rectangle,draw, fill = blue!20!white, minimum width=0.25cm,minimum height=1.5cm,align=center] at  (\x+0,2.1) (first) {};
    
    \node[rectangle,draw, fill = red!20!white, minimum width=0.25cm,minimum height=0.8cm,align=center] at  (\x+0,0.95) (inp) {};

\node (00) at (\x-1.2,3){};

\ifthenelse{\x=\xa}{\draw [decorate,
    decoration = {brace}] (\x-0.25,1.35) --  (\x-0.25,2.85)
    node[pos=0.75,left =8pt, black, rotate = 90]{};
    \draw [->,color=black,line width = 0.5mm] (00) |- (\x-0.4,2.1)
 node[pos=0, yshift = 5pt,black]{ \footnotesize  \a};}
    {\ifthenelse{\x=\xb}{\draw [decorate,
    decoration = {brace}] (\x-0.25,1.35) --  (\x-0.25,2.85);
    \draw[->,color=black,line width = 0.5mm](\x-1,2.1) -- (\x-0.35,2.1) node[pos=0, xshift = -13pt, yshift = 20pt, black]{\a};}{  \draw [decorate,
    decoration = {brace}] (\x-0.25,1.35) --  (\x-0.25,2.85);
    \draw[->,color=black,line width = 0.5mm](\x-1,2.1) -- (\x-0.35,2.1) node[pos=0, xshift = -13pt, yshift = 20pt, black]{\a};}}
    
\draw [decorate,
    decoration = {brace}] (\x-0.25,0.55) --  (\x-0.25,1.3)
  node[pos=0.5,left = 8pt, black]   (brace1) {};

\node (u0) at (\x-1.2,0.12){};
\draw [->,color=black,line width = 0.5mm] (u0) |- (\x-0.4,0.925)
 node[pos=0,yshift = -5pt, black]{\ifthenelse{\x=\xa}{\ia}{\ifthenelse{\x=\xb}{\ib}{\ic}} };

\draw [decorate,
    decoration = {brace,mirror}] (\x+0.25,0.55) --  (\x+0.25,2.85)
    node[pos=0.5,right = 0pt,black]{};
   
    \draw [->,color=black,line width = 0.5mm] (\x+0.4,1.7) -- (\x+0.8,1.7);
    
     \draw [->,color=black,line width = 0.5mm] (\x+1.6,1.7) -- (\x+2,1.7);
     
   \node[rectangle,draw, fill = blue!20!white, minimum width=0.25cm,minimum height=1.5cm,align=center] at  (\x+2.125,1.8) (first_out) {};
    
    \node[rectangle,draw, fill = black!20!white, minimum width=0.25cm,minimum height=0.2cm,align=center] at  (\x+2.125,0.95) (last) {};


   \ifthenelse{\x=\xc}{

    \draw [->,color=black,line width = 0.5mm] (\x+2.4,0.95) -- (\x+3.4 ,0.95);
    \node[yshift=7pt,xshift = 2pt] at (\x+3.55,0.95) {\tiny $\text{Last}(f^R\left(U,T-1\right))$};
   }
   {
       \draw [decorate,
    decoration = {brace,mirror}] (\x+2.35,1.05) --  (\x+2.35,2.55)
    node[pos=0.5,right = 0pt,black]{};
    \draw [->,color=black,line width = 0.5mm] (\x+2.5,1.8) -- (\x+3.4 ,1.8);
 \node[] at (\x+3.5,2.2) {\ifthenelse{\x=\xa}{\foa}{\ifthenelse{\x=\xb}{\fob}{\foc}}};
 }

    \node[rectangle,dashed,draw, minimum width=3.75cm,minimum height=3.5cm,align=center] at  (\x+1.08,1.7) (lay1) {};
}
 \node at (4.2,1.8) {$\ldots$}; 
 \node at (10.2,1.8) {$\ldots$};
\end{tikzpicture}}
\caption{An example of a recurrent model in $\brmodel{\cF}{T}$. The first $q-1$ dimensions of $f^R(U,t-1)$ is concatenated with $u^{(t)}$ to form the input at time $t$. The last dimension of $f^R(U,T-1)$ is taken to be the final output of the recurrent model.}
\label{fig:rec_model}
\end{figure}
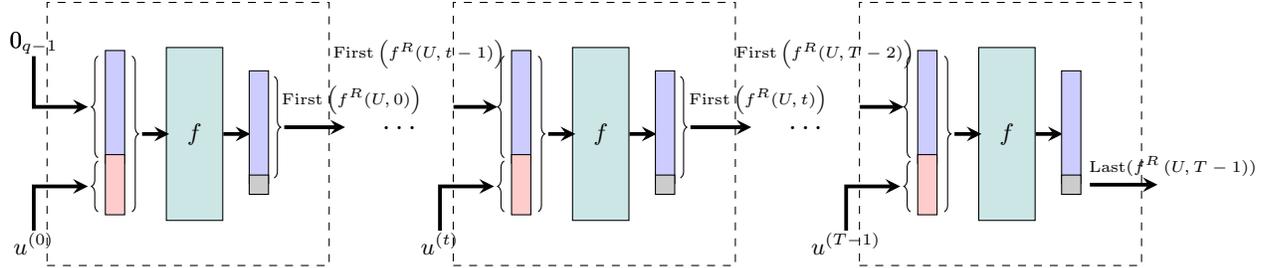


\subsection{PAC learning with ramp loss}
In this section we formulate the PAC learning model for classification with respect to the ramp loss. The use of ramp loss is natural for classification (see e.g.,~\citet{boucheron2005theory,bartlett2006convexity}) and the main features of the ramp loss that we are going to exploit are boundedness and Lipschitzness. We start by introducing the ramp loss.
\begin{definition}[Ramp Loss]
Let $f:\cX\rightarrow \bR$ be a hypothesis and let $\cD$ be a distribution over $\cX\times\cY$. Let $(x,y)\in\cX \times \cY$, where $\cY=\{-1,1\}$. The ramp loss of $f$ with respect to margin parameter $\gamma>0$ is defined as $l_{\gamma}(f,x,y) = {r_{\gamma}\left(-f(x).y\right)}$, where $r_{\gamma}$ is the ramp function defined by
\begin{equation*}
  r_{\gamma}(x)=\begin{cases}
    0 & x < -\gamma,\\
    1 + \frac{x}{\gamma} & -\gamma \leq x \leq 0 \\
    1 & x \geq 0.
    \end{cases}
\end{equation*}
\end{definition}

\begin{definition}[Agnostic PAC Learning with Respect to Ramp Loss]\label{def:pac_learn_ramp}
We say that a hypothesis class $\cF$ of functions from $\cX$ to $\bR$ is agnostic PAC learnable with respect to ramp loss with margin parameter $\gamma>0$ if there exists a learner $\cA$ and a function $m:(0,1)^2\rightarrow \bN$ with the following property: For every distribution $\cD$ over $\cX\times \{-1,1\}$ and every $\epsilon,\delta \in (0,1)$, if $S$ is a set of $m(\epsilon,\delta)$ i.i.d. samples from $\cD$, then with probability at least $1-\delta$ (over the randomness of $S$) we have
\begin{equation*}
      \expects{(x,y)\sim \cD}{l_{\gamma}\left(\cA(S),x,y\right)} \leq \inf_{f \in \cF}\expects{(x,y)\sim \cD}{l_{\gamma}\left(\cA(S),x,y\right)} + \epsilon.
\end{equation*}
\end{definition}
The \textit{sample complexity} of PAC learning $\cF$ with respect to ramp loss is denoted by $m_{\cF}(\epsilon,\delta)$, which is the minimum number of samples required for learning $\cF$ (among all learners $\cA$). The definition of agnostic PAC learning with respect to ramp loss works for any value of $\gamma$ and when we are analyzing the sample complexity we consider it to be a fixed constant.

\section{A lower bound for sample complexity of learning recurrent neural networks}\label{sec:lower_bound}
In this section, we consider the sample complexity of PAC learning sigmoid recurrent neural networks with respect to ramp loss. Particularly, we state a lower bound on the sample complexity of the class $\brmodel{\mnet{p_0}{p_k}{w}}{T}$ of all sigmoid recurrent neural networks with length $T$ that use multi-layer neural networks with $w$ weights as their recurring block.  The main message is that this sample complexity grows at least linearly with $T$. 
\begin{theorem}[Sample Complexity Lower Bound for Recurrent Neural Networks] \label{thm:lower_rnn}
For every $T\geq 3$ and $w\geq 19$ there exists a well-defined class $\cH_w = \brmodel{\mnet{p_0}{p_k}{w}}{T}$ and a universal constant $C>0$ such that for every $\epsilon,\delta \in (0,1/40)$ we have
\begin{equation*}
    m_{\cH_w}(\epsilon,\delta) \geq  C.\left(\frac{wT+\log(1/\delta)}{\epsilon^2}\right).
\end{equation*}
\end{theorem}
The proof of the above lower bound is based on a similar result due to~\cite{sontag1998vc}. However, the argument in~\cite{sontag1998vc} is for PAC learning with respect to 0-1 loss. To extend this result for the ramp loss, we construct a binary-valued class $\cF_w = \{f:f(U) = \text{sign}(h(U)),h\in\cH_w\}$ where $\sign{x}=1$ if $x\geq0$ and $\sign{x}=-1$ if $x<0$. We prove that every function $f\in\cF_w$ can be related to another function $h\in\cH_w$ such that the ramp loss of $h$ is almost equal to the zero-one loss of $f$. This is formalized in the following lemma, which is a key result in proving Theorem~\ref{thm:lower_rnn}. The proof of Theorem~\ref{thm:lower_rnn} and Lemma~\ref{lemma:ramp_to_zero_one} can be found in Appendix~\ref{app:lower_bound}.
\begin{lemma}\label{lemma:ramp_to_zero_one}
    Let $\cH_w = \brmodel{\mnet{p_0}{p_k}{w}}{T}$ be a well-defined class and let $\cF_w = \{f:[-1/2,1/2]^{p\times T}\to \{-1,1
\}\mid f(U) = {\sign{h(U)}},h\in\cH_w\}$. Then, for every distribution $\cD$ over $[-1/2,1/2]^{p\times T}\times \{-1,1\}$, $\eta>0$, and every function $f\in\cF_w$ there exists a function $h\in\cH_w$ such that $\expects{(U,y)\sim \cD}{l_{\gamma}\left(h,U,y\right)} \leq \expects{(U,y)\sim \cD}{l^{0-1}\left(f,U,y\right)} + \eta$ where $l^{0-1}(f,U,y)={\indicator{f(U)\neq y}}$.
\end{lemma}

\section{Noisy recurrent neural networks}

In this section, we will define classes of noisy recurrent neural networks. Let us first define the singleton Gaussian noise class, which contains a single additive Gaussian noise function.
\begin{definition}[The Gaussian Noise Class]
The $d$-dimensional noise class with scale $\sigma\geq 0$ is denoted by $\rv{\cG_{\sigma, d}}=\{\rv{g_{\sigma,d}}\}$. Here, $\rv{g_{\sigma,d}}:\bR^d\rightarrow\bR^d$ is a random function defined by $\rv{g_{\sigma,d}}(\rv{x})=\rv{x}+\rv{z}$, where $\rv{z}\sim \cN(\mathbf{0},\sigma^2 I_d)$. When it is clear from the context we drop $d$ and write $\rv{\cG_{\sigma}}=\{\rv{g_{\sigma}}\}$.
\end{definition}

The following is the noisy version of multi-layer networks in Definition~\ref{def:multinet}. Basically, Gaussian noise is composed (Definition~\ref{def:composition}) before each layer.
\begin{definition}[Noisy Multi-Layer Sigmoid Neural Networks]\label{def:noisy_multinet}
The class of all noisy multi-layer sigmoid networks with $w$ weights that take values in $[-1/2,1/2]^{p_0}$ as input and output values in $[-1/2,1/2]^{p_k}$ is defined by 
\begin{equation*}
   \rvmnet{p_0}{p_k}{w}{\sigma} = \bigcup \net{p_{k-1}}{p_k}    \circ \ldots \circ \rv{\cG_{\sigma}}\circ\net{p_1}{p_2}     \circ    \rv{\cG_{\sigma}}\circ\net{p_0}{p_1}\circ\rv{\cG_{\sigma}},
\end{equation*}
where $\sigma\geq 0$ is scale of the Gaussian noise and the union is taken over all choices of $(p_1, p_2, \ldots, p_{k-1})\in \bN^{k-1}$ that satisfy $\sum_{i=1}^{k} p_i.p_{i-1}=w$.
\end{definition}
Similar to the deterministic case, $\rvmnet{p_0}{p_k}{w}{\sigma}$ is said to be well-defined if the union is not empty (i.e., $p_0, p_k$ and $w$ are compatible). We can use Definition~\ref{def:rnet} to create recurrent versions of the above class. For example, $\brmodel{\rvmnet{p_0}{p_k}{w}{\sigma}}{T}$ is a class of recurrent (and random) hypotheses for sequence of length $T$ that use $\rvmnet{p_0}{p_k}{w}{\sigma}$ as their recurring block. Again, similar to the deterministic case, we say $\brmodel{\rvmnet{p_0}{p_k}{w}{\sigma}}{T}$ is well-defined if $p_0, p_k$ and $w$ are compatible and $\rvmnet{p_0}{p_k}{w}{\sigma}$ is well-defined.

\section{PAC learning noisy recurrent neural networks}\label{subsec:upper_bound}

In section~\ref{sec:lower_bound}, we established an $\Omega(T)$ lower bound on the sample complexity of learning recurrent networks (i.e., $\brmodel{\mnet{p_0}{p_k}{w}}{T}$). In this section, we consider a related class (based on noisy recurrent neural networks) and show that the dependence of sample complexity on $T$ is only $O(\log T)$. In particular, $\rv{\cG_{\sigma}}\circ\brmodel{\rvmnet{p_0}{p_k}{w}{\sigma}}{T}$ can be regarded as a (noisy) sibling of $\brmodel{\mnet{p_0}{p_k}{w}}{T}$. Since it is more standard to define PAC learnability for deterministic hypotheses, we define the deterministic version of the above class by derandomization\footnote{One can also define PAC learnability for a class of random hypotheses and get a similar result without taking the expectation. However, working with a deterministic class helps to contrast the result with that of Theorem~\ref{thm:lower_rnn}.}.

\begin{definition}[{Derandomization by Expectation}] Let $\cF$ be a class of (random) functions from $\bR^{p \times T}$ to $\bR^q$. The derandomization of a function class $\rv{\cF}$ by expectation is defined as $\cE(\rv{\cF})=\{h:\bR^{p \times T}\to \bR^{q}\mid h\left( u \right)=\expects{\rv{f}}{~\rv{f}\left(u\right)}, \rv{f}\in \rv{\cF}\}$.
\end{definition}
We show that, contrary to Theorem~\ref{thm:lower_rnn}, the sample complexity of PAC learning the (derandomized) class of noisy recurrent neural networks, $\cE(\rv{\cG_{\sigma}}\circ\brmodel{\rvmnet{p_0}{p_k}{w}{\sigma}}{T})$, grows at most logarithmically with $T$ while it still enjoys the same linear dependence on $w$. This is formalized in the following theorem (see Appendix~\ref{app:upper_bound} for a proof).
\begin{theorem}[Main Result]\label{thm:upper_rnn}
Let $\rv{\cQ_w}=\rv{\cG_{\sigma}}\circ\brmodel{\rvmnet{p_0}{p_k}{w}{\sigma}}{T}$ be any well-defined class and assume $T\in\bN,0<\sigma<1$, $\epsilon,\delta \in (0,1)$. Then the sample complexity of learning $\cH_w=\cE(\rv{\cQ_w})$ is upper bounded by
    \begin{equation*}
        m_{\cH_w}(\epsilon,\delta) = O\left(\frac{ w\log\left(\frac{wT}{\epsilon\sigma}\log\left(\frac{wT}{\epsilon\sigma}\right)\right) + \log\left(1
        /\delta\right)}{\epsilon ^2} \right) = \widetilde{O}\left(\frac{w\log\left(\frac{T}{\sigma}\right)+\log(1/\delta)}{\epsilon^2}\right),
    \end{equation*}
   where $\widetilde{O}$ hides {logarithmic factors}.
\end{theorem}

One feature of the above theorem is the mild logarithmic dependence on $1/\sigma$. Therefore, we can take $\sigma$ to be numerically negligible and still get a significantly smaller sample complexity compared to the deterministic case for large $T$. Note that adding such small values of noise would not change the empirical outcome of RNNs on finite precision computers. 

The milder (logarithmic) dependency on $T$ is achieved by a novel analysis that involves bounding the covering number of noisy recurrent networks with respect to the total variation distance. Also, instead of ``unfolding'' the network, we exploit the fact that the same function/hypothesis is being used recurrently. 
We also want to emphasize that the above bound does not depend on the norms of weights of the network. Achieving this is challenging, since a little bit of noise in a previous layer can change the output of the next layer drastically. The next few sections are dedicated to give a high-level proof of this theorem.
\section{Covering numbers: the classical view}\label{sec:covering_classic}
One of the main tools to derive sample complexity bounds for learning a class of functions is studying their covering numbers. In this section we formalize this classic tool.
\begin{definition}[Covering Number] 
Let $(\cX,\rho)$ be a metric space. A set $A\subset \cX$ is $ \epsilon$-covered by a set $ C\subseteq A$ with respect to $ \rho$, if for all $a\in A$ there exists $c\in C$ such that $\displaystyle \rho(a,c)\leq\epsilon$. We denote by $N(\epsilon,A,\rho)$ the cardinality of the smallest set $ C$ that $ \epsilon$-covers $A$ and we refer to is as the $\epsilon$-covering number of $A$ with respect to metric $\rho$.
\end{definition}
The notion of covering number is defined with respect to a metric $\rho$. We now give the definition of extended metrics, which we will use to define \textit{uniform} covering numbers. The extended metrics can be seen as measures of distance between two hypotheses on a given input set.
\begin{definition}[Extended Metrics]
Let $(\cX, \rho)$ be a metric space. Let $u=(a_1, \ldots, a_m),v=(b_1, \ldots, b_m)\in \cX^m$ for $m\in \bN$. The $\infty$-extended and $\ell_2$-extended metrics over $\cX^m$ are defined by $\rho^{\infty, m}(u,v)=\sup_{1\leq i \leq m} \rho(a_i,b_i)$ and $\rho^{\ell_2, m}(u,v)=\sqrt{\frac{1}{m}\sum_{i=1}^m (\rho(a_i,b_i))^2}$, respectively. We drop $m$ and use $\rho^\infty$ or $\rho^{\ell_2}$ if it is clear from the context.
\end{definition}
A useful property about extended metrics is that the $\infty$-extended metric always upper bounds the $\ell_2$-extended metric, i.e., $\rho^{\ell_2}(u,v)\leq\rho^{\infty}(u,v)$ for all $u,v\in\cX$. Based on the above definition of extended metrics, we define the uniform covering number of a hypothesis class with respect to $\|.\|_2$.

\begin{definition}[Uniform Covering Number with Respect to $\|.\|_2$]\label{def:unif_cn} 
Let $\cF$ be a hypothesis class of functions from $\cX$ to $\cY$. For a set of inputs $S=\{x_1, x_2, \ldots, x_m\}\subseteq \cX$, we define the restriction of $\cF$ to $S$ as $\cF_{|S}=\{(f(x_1), f(x_2), \ldots, f(x_m)) :f\in \cF\}\subseteq \cY^m$. The uniform $\epsilon$-covering numbers of hypothesis class $\cF$ with respect to $\|.\|_2^{\infty},\|.\|_2^{\ell_2}$ are denoted by $N_{U}(\epsilon,\cF,m,\|.\|_2^{\infty})$ and $N_{U}(\epsilon,\cF,m,\|.\|_2^{\ell_2})$ and are the maximum values of $ N(\epsilon,\cF_{|S},\|.\|_2^{\infty, m})$ and $ N(\epsilon,\cF_{|S},\|.\|_2^{\ell_2, m})$ over all $S\subseteq \cX$ with $|S|=m$, respectively.
\end{definition}

The following theorem connects the notion of uniform covering number with PAC learning. It converts a bound on the $\|.\|_2^{\ell_2}$ uniform covering number of a hypothesis class to a bound on the sample complexity of PAC learning the class; see Appendix~\ref{app:pac_cover} for a more detailed discussion.
\begin{theorem}\label{thm:pac_cover}
Let $\cF$ be a class of functions from $\cX$ to $\bR$. Then there exists an algorithm $\cA$ with the following property: For every distribution $\cD$ over $\cX\times \{-1,1\}$ and every $\epsilon,\delta \in (0,1)$, if $S$ is a set of $m$ i.i.d. samples from $\cD$, then with probability at least $1-\delta$ (over the randomness of $S$),
\begin{equation*}
\begin{aligned}
       & \expects{(x,y)\sim \cD}{l_{\gamma}\left(\cA(S),x,y\right)}\\
       & \leq \inf_{f\in \cF}\expects{(x,y)\sim \cD}{l_{\gamma}\left(f,x,y\right)}+16\epsilon + \frac{24}{\sqrt{m}}\sqrt{\ln N_U(\gamma\epsilon,\cF,m,\|.\|_2^{\ell_2})}+ 6\sqrt{\frac{\ln(2/\delta)}{2m}}.
\end{aligned}
\end{equation*}
Moreover, the algorithm that returns the function with the minimum error on $S$ satisfies the above property (i.e., Algorithm $\cA$ such that $\cA(S) = \arg\min_{f\in\cF}\frac{1}{|S|}\sum_{(x,y)\in S}{l_{\gamma}\left(f,x,y\right)}$).
\end{theorem}

\section{Total variation covers for random hypotheses}\label{sec:covering_random}
One idea to prove a generalization bound for noisy neural networks is to bound their covering numbers. However, noisy neural networks are random functions, and therefore their behaviours on a sample set cannot be directly compared. Instead, one can compare the output distributions of a random function on two sample sets. We therefore use the recently developed tools from~\citet{pour2022benefits} to define and study covering numbers for random hypotheses. These covering numbers are defined based on metrics between distributions. Specifically, our analysis is based on the notion of uniform covering number with respect to total variation distance.
\begin{definition}[Total Variation Distance]
Let $\mu$ and $\nu$ denote two probability measures over $\cX$ and let $\Omega$ be the Borel sigma-algebra over $\cX$. The TV distance between $\mu$ and $\nu$ is defined by
\begin{equation*}
    d_{TV}(\mu,\nu)=\sup_{B\in\Omega}|\mu(B)-\nu(B)|.
\end{equation*}
Furthermore, if $\mu$ and $\nu$ have densities $f$ and $g$ then
\begin{equation*}
    d_{TV}(\mu,\nu)=\sup_{B\in\Omega}\Big|\int_{B}(f(x)-g(x))dx\Big|=\frac{1}{2} \int_{\cX}\left|f(x)-g(x)\right|dx=\frac{1}{2}\|f-g\|_1.
\end{equation*}
\end{definition}
{For two random variables $\rv{x}$ and $\rv{y}$ with probability measures $\mu$ and $\nu$ we sometimes abuse the notation and write $d_{TV}(\rv{x},\rv{y})$ instead of $d_{TV}(\mu,\nu)$.}
For example, we write $d_{TV}(\rv{f_1}(\rv{x}),\rv{f_2}(\rv{x}))$ in order to refer to the Total Variation (TV) distance between pushforwards of $\rv{x}$ under mappings $\rv{f_1}$ and $\rv{f_2}$. We also write $\allowdisplaybreaks d_{TV}^{\infty,m}\left(\left(\rv{f_1}(\rv{x_1}),\ldots,\rv{f_1}(\rv{x_m})\right), \left(\rv{f_2}(\rv{x_1}),\ldots,\rv{f_2}(\rv{x_m})\right)\right)$ to refer to the extended TV distance between mappings of the set $S=\{\rv{x_1},\ldots,\rv{x_m}\}$ by $\rv{f_1}$ and $\rv{f_2}$.
We use the extended total variation distance to define the uniform covering number for classes of random hypotheses.
\begin{definition}[Uniform Covering Number for Classes of Random Hypotheses]\label{def:unif_cnr}
Let $\rv{\cF}$ be a class of random hypotheses from $\rv{\cX}$ to $\rv{\cY}$. For a set of random variables $\rv{S}=\{\rv{x_1}, \rv{x_2}, \ldots, \rv{x_m}\}\subseteq \rv{\cX}$, the restriction of $\rv{\cF}$ to $\rv{S}$ is defined as $\displaystyle \rv{\cF}_{|\rv{S}}=\{(\rv{f}(\rv{x_1}), \rv{f}(\rv{x_2}), \ldots, \rv{f}(\rv{x_m})) :\rv{f}\in \rv{\cF}\}\subseteq{\rv{\cY}}^m$. 
Let $\Gamma\subseteq \rv{\cX}$. The uniform $\epsilon$-covering numbers of $\rv{\cF}$ with respect to $\Gamma$ and  $d_{TV}^{\infty}$ is defined by
\begin{equation*}
\begin{aligned}
    N_{U}(\epsilon,\rv{\cF},m,d_{TV}^{\infty}, \Gamma) = \sup_{S\subseteq \Gamma, |S|=m} N(\epsilon,\rv{\cF}_{|\rv{S}},d_{TV}^{\infty, m}).
\end{aligned}
\end{equation*}
\end{definition} 
Some hypothesis classes that we analyze (e.g., single-layer noisy neural networks) may have ``global'' total variation covers that do not depend on $m$. This will be addressed with the following notation:
$
    N_{U}(\epsilon,\rv{\cF},\infty,\rho^{{\infty}}, \Gamma) = \lim_{m\to \infty} N_{U}(\epsilon,\rv{\cF},m,\rho^{{\infty}}, \Gamma).
$
The set $\Gamma$ in Definition~\ref{def:unif_cnr} is used to define the input domain for which we want to find the covering number of a class of random hypotheses. For instance, some of the covers that we see are derived with respect to inputs with bounded domain or some need the input to be first smoothed by Gaussian noise. In this paper, we will be working with the following choices of $\Gamma$
\begin{itemize}
    \item[--] $\Gamma=\rv{\cX_d}$ and $\Gamma=\rv{\cX_{B,d}}$: the set of all random variables defined over $\bR^d$ and $[-B,B]^d$, respectively, that admit a generalized density function. For example, we use $\rv{\cX_{0.5,d}}$ to address the set of random variables in $[-1/2,1/2]^d$.
     \item[--] $\Gamma=\rv{\Delta_{p\times T}}=\{\rv{U} \mid \rv{U} = \begin{bmatrix}
         \rv{\delta_{u^{(0)}}}&\ldots&\rv{\delta_{u^{(T-1)}}}
     \end{bmatrix}\transpose, u^{(i)}\in \bR^p\}$ and $\Gamma=\rv{\Delta_{B,p\times T}}=\{\rv{U} \mid \rv{U} = \begin{bmatrix}
         \rv{\delta_{u^{(0)}}}&\ldots&\rv{\delta_{u^{(T-1)}}}
     \end{bmatrix}\transpose, u^{(i)}\in [-B,B]^p\}$, where $\rv{\delta_{u^{(i)}}}$ is the random variable associated with Dirac delta measure on $u^{(i)}$. Note that $\rv{\Delta_{B,p\times T}}\subset \rv{\Delta_{p\times T}}$.
    \item[--] $\Gamma=\rv{\cG_{\sigma,d}} \circ \rv{\cX_{B,d}}=\{\rv{g_{\sigma,d}}(\rv{x})\mid \rv{x}\in \rv{\cX_{B, d}}\}$: all members of $\rv{\cX_{B,d}}$ after being ``smoothed'' by adding (convolving the density with) Gaussian noise. 
\end{itemize}

We mentioned in Section~\ref{sec:covering_classic} that a bound on the $\|.\|_2^{\ell_2}$ uniform covering number can be connected to a bound on sample complexity of PAC learning. We now show that a bound on $d_{TV}^{\infty}$ covering number of a class of random hypotheses can be turned into a bound on the $\|.\|_2^{\ell_2}$ covering number of its derandomized version and, thus, PAC learning it.
\begin{theorem}[$\|.\|_2^{\ell_2}$ Cover of $\cE(\cF)$ From $d_{TV}^{\infty}$ Cover of $\cF$ \citep{pour2022benefits}]\label{thm:cover_derandomization}
    Let $\rv{\cF}$ be a class of functions from $\bR^{p \times T}$ to $[-B,B]^q$. Then for every $\epsilon>0$ and $m\in\bN$ we have
   \begin{equation*}
   \begin{aligned}
     & N_U(2B\epsilon\sqrt{q},\cE(\rv{\cF}),m,\|.\|_2^{\ell_2}) \leq N_U(\epsilon,\rv{\cF},m,d_{TV}^{\infty},\rv{\Delta_{p\times T}})\leq  N_U(\epsilon,\rv{\cF},\infty,d_{TV}^{\infty},\rv{\Delta_{p\times T}}).
   \end{aligned}
\end{equation*}
\end{theorem}
\section{Bounding the covering number of recurrent models}\label{sec:proof_scheme}
In Section~\ref{sec:covering_classic}, we mentioned that finding a bound on covering number of a hypothesis class is a standard approach to bound its sample complexity. In the previous section, we introduced a new notion of covering number with respect to total variation distance that was developed by \citet{pour2022benefits}. We showed how this notion can be related to PAC learning for classes of random hypotheses. In the following, we give an overview of the techniques used to find a bound on the $d_{TV}^{\infty}$ covering number of the class of noisy recurrent models. We also discuss why this bound results in a sample complexity that has a milder logarithmic dependency on $T$, compared to bounds proved by ``unfolding'' the recurrence and replacing the recurrent model with the $T$-fold composition.

One advantage of analyzing the uniform covering number with respect to TV distance is that it comes with a useful composition tool. The following theorem basically states that when two classes of hypotheses have bounded TV covers, their composition class has a bounded cover too.
Note that such a result does not hold for the usual definition of covering number (e.g., Definition ~\ref{def:unif_cn}); see~\citet{pour2022benefits} for details.
\begin{theorem}[TV Cover for Composition of Random Classes, Lemma~18 of \citet{pour2022benefits}]\label{thm:compose_tv}
    Let $\rv{\cF}$ be a class of random hypotheses from $\bR^d$ to $\bR^p$ and $\rv{\cH}$ be a class of random hypotheses from $\bR^p$ to $\bR^q$. For any $\epsilon_1,\epsilon_2>0$ and $m\in\bN$, denote $N_1=N_U\left(\epsilon_1,\rv{\cF},m,d_{TV}^{{\infty}}, \rv{\cX_{d}}\right)$. Then we have,
    \begin{equation*}
         N_U\left(\epsilon_1+\epsilon_2,\rv{\cH}\circ\rv{\cF},m,d_{TV}^{{\infty}}, \rv{\cX_{d}}\right) \leq N_U\left(\epsilon_2,\rv{\cH},{m}N_1,d_{TV}^{{\infty}}, \rv{\cX_{p}}\right).N_1.
    \end{equation*}
\end{theorem}

An approach to bound the TV uniform covering number of a recurrent model $\brmodel{\rv{\cF}}{T}$ is to consider it as the $T$-fold composition $\rv{\cF}\circ \rv{\cF}\ldots\circ \rv{\cF}$. One can then use a similar analysis to that of~\citet{pour2022benefits} to bound the covering number of the $T$-fold composition.
 Unfortunately, this approach fails to capture the fact that a \textit{fixed} function $\rv{f}\in\rv{\cF}$ is applied recursively, and therefore results in a sample complexity bound that grows at least linearly with $T$.

Instead, we take another approach to bound the covering number of recurrent models. Intuitively, we notice that any function in the $T$-fold composite class $\rv{\cF}\circ\ldots\circ\rv{\cF}=\{\rv{f_1}\circ\ldots \circ\rv{f_T}\mid \rv{f_1},\ldots,\rv{f_T}\in\rv{\cF}\}$ is determined by $T$ functions from $\rv{\cF}$. On the other hand, any function in $\brmodel{\rv{\cF}}{T} = \left\{\rv{h}\mid \rv{h}(U)=\last{\rv{f}^R(U,T-1)}\right\}$ is only defined by one function in $\rv{\cF}$ and the capacity of this class must not be as large as the capacity of $\rv{\cF}\circ\ldots\circ\rv{\cF}$. 
Interestingly, data processing inequality for total variation distance (Lemma~\ref{lemma:DPI}) suggests that if two functions $\rv{f}$ and $\rv{\hat{f}}$ are ``globally'' close to each other with respect to TV distance (i.e., $d_{TV}(\rv{f}(\rv{x}),\rv{\hat{f}}(\rv{x}))\leq \epsilon$ for every $\rv{x}$ in the domain), then $d_{TV}(\rv{f}(\rv{f}(\rv{x})),\rv{\hat{f}}(\rv{\hat{f}}(\rv{x})))\leq 2\epsilon$ (i.e., $\rv{f}\circ\rv{f}$ and $\rv{\hat{f}}\circ\rv{\hat{f}}$ are also close to each other). By applying the data processing inequality recursively, we can see that for the $T$-fold composition we have $d_{TV}(\rv{f}\circ\ldots\circ\rv{f}(\rv{x}),\rv{\hat{f}}\circ\ldots\circ\rv{\hat{f}}(\rv{x}))\leq \epsilon T$. The above approach results in the following theorem which bounds the $\epsilon$-covering number of a noisy recurrent model with respect to TV distance by the $(\epsilon/T)$-covering number of its recurring class.
Intuitively, this theorem helps us to bound the covering number of noisy recurrent models using the bounds obtained for their non-recurrent versions. Here, Gaussian noise is added to both the input of the model (i.e., $\rv{\cF_{\sigma}}=\rv{\cF}\circ\rv{\cG_{\sigma}}$) and the output of the model (by composing with $\rv{\cG_{\sigma}}$).
\begin{theorem}[TV Covering Number of $\rv{\cG_{\sigma}}\circ\brmodel{\rv{\cF_{\sigma}}}{T}$ From $\rv{\cG_{\sigma}}\circ\rv{\cF_\sigma}$]\label{thm:cover_recurrent_model}
Let $s,p,q\in \bN$ such that $s=p+q-1$.
Let $\rv{\cF}$ be a class of functions from $\rv{\cX_{B,s}}$ to $\rv{\cX_{B,q}}$ and denote by $\rv{\cF_{\sigma}} = \rv{\cF}\circ\rv{\cG_{\sigma,s}}$ the class of its composition with noise. Then we have
\begin{equation*}
N_U\left(\epsilon,\rv{\cG_\sigma}\circ\brmodel{\rv{\cF_{\sigma}}}{T},\infty,d_{TV}^{\infty},\rv{\Delta_{B,p\times T}}\right) \leq N_U\left(\epsilon/T,\rv{\cG_{\sigma,q}}\circ\rv{\cF_{\sigma}},\infty,d_{TV}^{\infty},\rv{\cX_{B,s}}\right).
\end{equation*}
\end{theorem}

For using this theorem, one needs to have a finer $\epsilon/T$-cover for the recurring class. As we will see in the next section, this will translate into a mild logarithmic sample complexity dependence on $T$.

\subsection{Covering noisy recurrent networks}
An example of $\rv{\cF_{\sigma}}$ is the class $\rvmnet{p_0}{p_k}{w}{\sigma}$ of well-defined noisy multi-layer networks (Definition~\ref{def:noisy_multinet}). Theorem~\ref{thm:cover_recurrent_model} suggests that a bound on the covering number of $\rv{\cG_{\sigma}}\circ\brmodel{\rvmnet{p_0}{p_k}{w}{\sigma}}{T}$ can be found from a bound for $\rv{\cG_{\sigma}}\circ\rvmnet{p_0}{p_k}{w}{\sigma}$. We use the following theorem as a bound for the class of single-layer noisy sigmoid networks together with theorem~\ref{thm:compose_tv} to bound the covering number of $\rv{\cG_{\sigma}}\circ\rvmnet{p_0}{p_k}{w}{\sigma}$ (see Appendix~\ref{app:upper_bound}, Theorem~\ref{thm:cover_multi}). 
\begin{theorem}[A TV Cover for Single-Layer Noisy Neural Networks, Theorem 25 of \citet{pour2022benefits}]\label{thm:cover_single} 
For every $p,d\in \bN, \epsilon>0, \sigma<5d/\epsilon$ we have
\begin{align*}
   &\log N_U(\epsilon,\rv{\cG_{\sigma,p}}\circ\net{d}{p},\infty,d_{TV}^{\infty},\rv{\cG_{\sigma,d}}\circ\rv{\cX_{0.5,d}})\leq  p(d+1)\log\left(30\frac{d^{5/2}\sqrt{\ln\left(\frac{5d-\epsilon\sigma}{\epsilon\sigma}\right)}}{\epsilon^{3/2}\sigma^2}\ln\left(\frac{5d}{\epsilon\sigma}\right)\right).
\end{align*}
\end{theorem}

Interestingly, the above bound (on the logarithm of the covering number) is logarithmic with respect to $1/\epsilon$. We will extend this result to multi-layer noisy networks, and then apply Theorem~\ref{thm:cover_recurrent_model} 
to obtain the following bound on the covering number noisy recurrent neural networks. Crucially, the dependency (of the logarithm of the covering number) on $T$ is only logarithmic.
\begin{theorem}[A TV Covering Number Bound for Noisy Sigmoid Recurrent Networks]\label{thm:cover_noisy_rnn}
    Let $T\in\bN$. For every $\epsilon,\sigma\in(0,1)$ and every well-defined class $\brmodel{\rvmnet{p_0}{p_k}{w}{\sigma}}{T}$ we have
    \begin{equation*}
    \begin{aligned}
         &\log N_U\left(\epsilon,\rv{\cG_\sigma}\circ\brmodel{\rvmnet{p_0}{p_k}{w}{\sigma}}{T},\infty,d_{TV}^{\infty},\rv{\Delta_{0.5,p\times T}}\right)\\
         & = O\left( w\log\left(\frac{wT}{\epsilon\sigma}\log\left(\frac{wT}{\epsilon\sigma}\right)\right)\right) = \widetilde{O}\left(w\log\left(\frac{T}{\epsilon\sigma}\right)\right).
    \end{aligned}
    \end{equation*}
\end{theorem}

Finally, we turn the above bound into a $\|.\|_2^{\ell_2}$ covering number bound for the derandomized function $\cE\left({\rv{\cG_\sigma}\circ\brmodel{\rvmnet{p_0}{p_k}{w}{\sigma}}{T}}\right)$ by an application of Theorem~\ref{thm:cover_derandomization}. We then upper bound the sample complexity by the logarithm of covering number (see Theorem~\ref{thm:pac_cover}) and conclude Theorem~\ref{thm:upper_rnn}.

\section{Conclusion and future work}\label{sec:conclusion}
In this paper, we considered the class of noisy recurrent networks, where independent Gaussian noise is added to the output of each neuron. We proved that the sample complexity of learning this class depends logarithmically on the sequence length $T$. To prove this result, we invoked the tools provided by \citet{pour2022benefits} to analyze the covering number of this class of \textit{random} functions with respect to total variation distance. We showed how these tools make it possible to use the fact that the same \textit{fixed} function/network is being applied recursively, in contrast to techniques that are based on ``unfolding'' the recurrence. We also proved a lower bound for the non-noisy version of the same recurrent network which grows at least linearly with $T$. This exponential gap comes with only a mild logarithmic dependence on $1/\sigma$, where $\sigma$ is the standard deviation of the Gaussian noise added to the neurons. For numerically negligible amounts of noise (e.g., smaller than the precision of the computing device) the noisy and non-noisy network become effectively similar when implemented and our results show that one can break the $\Omega(T)$ barrier by only inducing a mild logarithmic dependence on $1/\sigma$. 

Our results are derived for sigmoid (basically bounded, monotone, and Lipschitz) activation functions. It is open whether such results can be proved for unbounded activation functions such as ReLU. Also, our results are theoretical and we leave empirical evaluations on the performance of noisy networks to future work. 
\section{More on related work}\label{app:related}
There are a family of approaches that aim to analyze the generalization in neural networks by bounding the VC-dimension of networks~\citep{baum1988size, maass1994neural,goldberg1995bounding,vidyasagar1997theory, sontag1998vc,koiran1998vapnik,bartlett1998almost,bartlett2003vapnik, bartlett2019nearly}. These approaches result in generalization bounds that are dependent on the number of parameters. Another family of approaches are aimed at obtaining generalization bounds that are dependent on the norms of the weights and Lipschitz continuity properties of the network~\citep{bartlett1996valid,anthony1999neural, zhang2002covering,pmlr-v40-Neyshabur15, bartlett2017spectrally,neyshabur2017pac,golowich2018size,arora2018stronger,nagarajan2018deterministic,long2020size}. It has been observed that these generalization bounds are usually vacuous in practice. One speculation is that the implicit bias of gradient descent~\citep{gunasekar2017implicit,arora2019implicit,ji2020gradient,chizat2020implicit,ji2021characterizing} can lead to benign overfitting~\citep{belkin2018overfitting,belkin2019does,bartlett2020benign,bartlett2021deep}. It has also been conjectured that uniform convergence theory may not be able to fully capture the performance of neural networks in practice~\citep{nagarajan2019uniform, zhang2021understanding}. It has been shown that there are data-dependent approaches that can achieve non-vacuouys bounds \citep{DR17,zhou2019non,negrea2019information}. There are also other approaches that are independent of data \citep{arora2018stronger}; see \citet{pour2022benefits} for more details. 

Adding different types of noise such as dropout noise \citep{JMLR:v15:srivastava14a}, DropConnect \citep{wan2013regularization}, and Denoising AutoEncoders~\citep{vincent2008extracting} are shown to be helpful in training neural networks. \citet{WANG2019177} and \citet{gao2016dropout} theoretically analyze the generalization under dropout noise. More recently, \citet{pour2022benefits} developed a framework to study the generalization of classes of noisy hypotheses and showed that adding noise to the output of neurons in a network can be helpful in generalization. \citet{jim1996analysis} showed that additive and multiplicative noise can help speed up the convergence of RNNs on local minima surfaces. Recently, \citet{lim2021noisy} showed that noisy RNNs are more stable and robust to input perturbations by formalizing the regularization effects of noise.

Another line of work focuses on the generalization of neural networks that are trained with Stochastic Gradient Descent (SGD) or its noisy variant Stochastic Gradient Langevin Descent (SGLD)~\citep{russo2016controlling, xu2017information, russo2019much, steinke2020reasoning,raginsky2017non,haghifam2020sharpened,neu2021information}. \citet{zhao2020rnn} analyze the memory properties of recurrent networks and how well they can remember the input sequence. \citet{Tu2020Understanding} study the generalization of RNN by analyzing the Fisher-Rao norm of weights, which they obtain from the gradients of the network. They offer generalization bounds that can potentially become polynomial in $T$. \citet{allen2019can} analyze the change in output through the dynamics of training RNNs and prove generalization bounds for recurrent networks that are again polynomial in $T$.


\bibliography{refs}
\appendix
\section{Miscellaneous facts}
\begin{lemma}[Data Processing Inequality for TV Distance]\label{lemma:DPI}
Given two random variables $\rv{x_1},\rv{x_2} \in \rv{\cX}$, and a (random) Borel function $f:\cX\rightarrow\cY$, 
\begin{equation*}
    d_{TV}(f(\rv{x_1}),f(\rv{x_2}))\leq d_{TV}(\rv{x_1},\rv{x_2}).
\end{equation*}
\end{lemma}

\begin{lemma}\label{lemma:coupling_tv}
  Let $\rv{x},\rv{y} \in \rv{\cX}$ be two random variables with probability measures $\mu$ and $\nu$. Denote by $\Pi(\mu,\nu)$ the set of all their couplings. Then, there exists $\pi^*\in \Pi(\mu,\nu)$ such that $\probs{\pi^*}{\rv{x}\neq \rv{y}} = d_{TV}(\mu,\nu)$, where the subscript $\pi^*$ signals that the probability law is associated with the coupling $\pi^*$. Moreover, for any coupling $\pi \in \Pi(\mu,\nu)$ we have $\probs{\pi}{\rv{x}\neq\rv{y}} \geq d_{TV}(\mu,\nu)$.
\end{lemma}

We use the following two lemmas to reason about the covering number of our recurrent model when we take the first dimensions of the output at each time $t$ and when we concatenate new inputs with the outputs. The first lemma states that if two random variables are close to each other with respect to total variation distance, then they are still close after the applications of the $\first{.}$ and $\last{.}$ functions.

\begin{lemma}[From TV of Random Variable to TV of First and Last]\label{lemma:tv_to_tv_first}
Let $\rv{x_1},\rv{x_2} \in \bR^d$ be two random variables. We have
\begin{equation*}
\begin{aligned}
   & d_{TV}\left( \first{\rv{x_1}},\first{\rv{x_2})}\right)\leq d_{TV}\left( \rv{x_1},\rv{x_2}\right) ,\\
   & d_{TV}\left( \last{\rv{x_1}},\last{\rv{x_2})}\right)\leq d_{TV}\left( \rv{x_1},\rv{x_2}\right). 
\end{aligned}
\end{equation*}
\end{lemma}
\begin{proof}
We know that $\first{.}$ and $\last{.}$ are functions from $\bR^d$ to $\bR^{d-1}$. Therefore we can apply Lemma~\ref{lemma:DPI} and conclude the result.
\end{proof}

The next lemma is used to bound the total variation distance between two random variables after being concatenated with the input at time $t$. In that case, we let $\rv{x_1}$ and $\rv{x_2}$ in the lemma to be $\first{f^R\left(U,t-1\right)}$ and $\first{\hat{f}^R\left(U,t-1\right)}$, which are in $\rv{\cX_{p_k-1}}$. We also let $\rv{y}$ be $u^{(t)}\in\rv{\Delta_d}$, which is the input at time $t$.
\begin{lemma}[From TV of Random Variable to TV of Concatenation]\label{lemma:tv_to_tv_concatenate}
    Let $\rv{x_1},\rv{x_2}$ be random variables in $\rv{\cX_d}$. Further, let $\rv{y}$ a random variable in $\rv{\Delta_d}$. If we have $d_{TV}\left(\rv{x_1},\rv{x_2}\right)\leq\epsilon$, then
    \begin{equation*}        d_{TV}\left(\begin{bmatrix}
        \rv{x_1}&\rv{y}
    \end{bmatrix}\transpose,\begin{bmatrix}
    \rv{x_2}&\rv{y}
\end{bmatrix}\transpose\right)\leq \epsilon.
    \end{equation*}
\end{lemma}
\begin{proof}
    Let $y\in\rv{\Delta_d}$ be the random variable with Dirac delta measure on $y_0$. From Lemma~\ref{lemma:coupling_tv} we know that there exists a maximal coupling $\pi^*$ of $\rv{x_1}$ and $\rv{x_2}$ such that $d_{TV}(\rv{x_1},\rv{x_2}) = \probs{\pi^*}{\rv{x_1}\neq \rv{x_2}}$ and denote the density associated with $\bP_{\pi^*}$ by $f^*$. Let $\gamma$ be a coupling of $\begin{bmatrix}\rv{x_1}&\rv{y_1}\end{bmatrix}\transpose$ and $\begin{bmatrix}\rv{x_2}&\rv{y_2}\end{bmatrix}\transpose$ such that 
    \begin{equation*}
        \hat{f}\left({\begin{bmatrix}x_1& y_1\end{bmatrix}\transpose,\begin{bmatrix}x_2&y_2\end{bmatrix}\transpose}\right) = \begin{cases}
        f^*\left(x_1,x_2\right) & y_1=y_2=y_0,\\
        0 & \mbox{otherwise}.\end{cases}
    \end{equation*}
We can easily verify that $\gamma$ is a valid coupling. Denote by $f_{x_1y}$ the density of the random variable $\begin{bmatrix}
   \rv{x_1} & y 
\end{bmatrix}\transpose$. We know that 
 \begin{equation*}
        f_{\rv{x_1y}}\left(\begin{bmatrix}x_1& y_1\end{bmatrix}\transpose\right) = \begin{cases}
        f_{\rv{x_1}}\left(x_1\right) & y=y_0,\\
        0 & \mbox{otherwise},\end{cases}
    \end{equation*}
where $f_{\rv{x_1}}$ is the density function of the random variable $\rv{x_1}$. We can observe that density associated with the marginal of $\gamma$ would be the same as the density of the marginal of $\pi^*$ at points where $y=y_0$ and it is zero otherwise. On the other hand, we know that $\pi^*$ is a valid coupling of $\rv{x_1}$ and $\rv{x_2}$ and therefore the density of its marginal is $f_{\rv{x_1}}$. This concludes that the density of the marginal of $\gamma$ is indeed $f_{\rv{x_1y}}$. We can show the similar thing for the other marginal, which concludes that $\gamma$ is a valid coupling.

Therefore, from Lemma~\ref{lemma:coupling_tv} we can write that
    \begin{equation*}
        \begin{aligned}
&d_{TV}\left(\begin{bmatrix}\rv{x_1}&\rv{y}\end{bmatrix}\transpose,\begin{bmatrix}\rv{x_2}&\rv{y}\end{bmatrix}\transpose\right) \leq \probs{\gamma}{\begin{bmatrix}\rv{x_1}&\rv{y}\end{bmatrix}\transpose\neq\begin{bmatrix}\rv{x_2}&\rv{y}\end{bmatrix}\transpose}\\
& \leq \int_{\begin{bmatrix}
    x_1\\y
\end{bmatrix}\neq \begin{bmatrix}
    x_2\\y
\end{bmatrix}}
\hat{f}\left({\begin{bmatrix}x_1& y\end{bmatrix}\transpose,\begin{bmatrix}x_2&y\end{bmatrix}\transpose}\right) \leq \int_{\substack{x_1\neq x_2,\\ y = y_0}}
\hat{f}\left({\begin{bmatrix}x_1& y\end{bmatrix}\transpose,\begin{bmatrix}x_2&y\end{bmatrix}\transpose}\right)
 \\
 &  \leq \int_{\substack{x_1\neq x_2,\\ y = y_0}} f^*\left(x_1,x_2\right)
 \leq \probs{\pi^*}{\rv{x_1}\neq\rv{x_2}} = d_{TV}\left(\rv{x_1},\rv{x_2}\right)\leq \epsilon.
        \end{aligned}
    \end{equation*}
\end{proof}
\section{Proof of lower bound}\label{app:lower_bound}
In order to prove Theorem~\ref{thm:lower_rnn}, we first need to give the definition of PAC Learning with respect to $0-1$ loss.
\begin{definition}[Agnostic PAC Learning with Respect to 0-1 Loss]
We say that a hypothesis class $\cF$ of functions from $\cX$ to $\bR$ is agnostic PAC learnable with respect to $0-1$ loss if there exists a learner $\cA$ and a function $m^{0-1}:(0,1)^2\rightarrow \bN$ with the following property: For every distribution $\cD$ over $\cX\times \{-1,1\}$ and every $\epsilon,\delta \in (0,1)$, if $S$ is a set of $m(\epsilon,\delta)$ i.i.d. samples from $\cD$, then with probability at least $1-\delta$ (over the randomness of $S$) we have
\begin{equation*}
    \expects{(x,y)\sim D}{l^{0-1}(\cA(S), x,y)} \leq \inf_{f \in \cF} \expects{(x,y)\sim D}{l^{0-1}(f, x,y)} + \epsilon.
\end{equation*}
\end{definition}
Same as Definition~\ref{def:pac_learn_ramp}, we denote by $m^{0-1}_{\cF}(\epsilon,\delta)$ the \textit{sample complexity} of PAC learning $\cF$ with respect to $0-1$ loss, which is the minimum number of samples required for learning $\cF$ among all learners $\cA$. 

Before proving Theorem~\ref{thm:lower_rnn}, we first prove Lemma~\ref{lemma:ramp_to_zero_one}, which, as mentioned before, is a core part of the proof. We state the lemma once more for completeness.
\begin{lemma}
Let $\cH_w = \brmodel{\mnet{p_0}{p_k}{w}}{T}$ be a well-defined class and let $\cF_w = \{f:[-1/2,1/2]^{p\times T}\to \{-1,1
\}\mid f(U) = {\sign{h(U)}},h\in\cH_w\}$. Then, for every distribution $\cD$ over $[-1/2,1/2]^{p\times T}\times \{-1,1\}$, $\eta>0$, and every function $f\in\cF_w$ there exists a function $h\in\cH_w$ such that $\expects{(U,y)\sim \cD}{l_{\gamma}\left(h,U,y\right)} \leq \expects{(U,y)\sim \cD}{l^{0-1}\left(f,U,y\right)} + \eta$ where $l^{0-1}(f,U,y)={\indicator{f(U)\neq y}}$.
\end{lemma}
\begin{proof}
We know that $\cH_w = \{h:\bR^{p\times T}\to[-1/2,1/2]\mid h(u) = \last{b^R\left(U,T-1\right)}, b\in \mnet{p_0}{p_k}{w}\}$. Similarly, $\cF_w = \{f:\bR^{p\times T}\to \{-1,1\}\mid f(u) = \sign{\last{b^R\left(U,T-1\right)}}, b\in \mnet{p_0}{p_k}{w}\}$. Fix a distribution $\cD$ over $[-1/2,1/2]^{p\times T} \times \{-1,1\}$. Define 
\begin{equation*}
    z = \min_b\,\argmax{0<x<\frac{1}{2}} \,\prob{\left|\last{b^R(U,T-1)}\right|  \geq x} \geq 1-\eta,
    \end{equation*}
where the minimum is taken over all well-defined multi-layer neural networks $b$ in $\mnet{p_0}{p_k}{w}$. The last dimension of function $b$ is in $[-1/2,1/2]$ and, intuitively, $z$ is the largest possible value such that $\prob{-z<\last{b^R(U,T-1)}<z} < \eta$.

Let $f$ be any function in $\cF_w$ and let $b = b_{k-1}\circ\ldots \circ b_0$ be the $k$-layer network associated with $f$ where $b_i $'s are single-layer sigmoid neural networks, i.e., $ f(U) = \sign{\last{b^R\left(U,T-1\right)}}$. Let $W_{k-1} = \begin{bmatrix} v_1 & \ldots & v_{p_k}\end{bmatrix}\transpose$ be the weight matrix associated with $b_{k-1}$. Denote by $\hat{W}_{k-1} = \begin{bmatrix} v_1 & \ldots & c.v_{p_k}\end{bmatrix}\transpose$ the matrix that is exactly the same as $W_{k-1}$ but every element in its last row is multiplied by $c = \phi^{-1}(\gamma)/\phi^{-1}(z)$. Note that $z>0$ and, therefore, $\phi^{-1}(z)>0$. Let $\hat{b}_{k-1}$ be the single-layer neural network that is defined by weight matrix $\hat{W}_{k-1}$, i.e., $\hat{b}_{k-1}(x) = \Phi\left(\hat{W}^{\top}_{k-1} x\right)$. Denote $\hat{b} = \hat{b}_{k-1}\circ\ldots \circ b_0$ and let $h(U) = \last{\hat{b}^R\left(U,T-1\right)}$ for any $U\in\bR^{p\times T}$. Clearly, $\hat{b}\in\mnet{p_0}{p_k}{w}$ and $h\in\cH_w$. We claim that $\expects{(U,y)\sim \cD}{l_{\gamma}\left(h,U,y\right)} \leq \expects{(U,y)\sim \cD}{l^{0-1}\left(f,U,y\right)} +\eta$. 
 We can write the definition of ramp loss as
\begin{equation}\label{eq:prob_ramp_loss}
\begin{aligned}
         & \expects{(U,y)\sim \cD}{l_{\gamma}\left(h,U,y\right)}=\expects{(U,y)\sim \cD}{r_{\gamma}\left(-h(U).y\right)} \\
         & =\expects{(U,y)\sim \cD}{r_{\gamma}\left(-h(U).y\right) \middle |\, \left|h(U)\right| \geq \phi\left(c.\phi^{-1}\left(z\right)\right)}. \prob{\left|h(U)\right| \geq \phi\left(c.\phi^{-1}\left(z\right)\right) }\\
         & + \expects{(U,y)\sim \cD}{r_{\gamma}\left(-h(U).y\right) \middle | \,\left|h(U)\right| < \phi\left(c.\phi^{-1}\left(z\right)\right)}. \prob{\left|h(U)\right| < \phi\left(c.\phi^{-1}\left(z\right)\right)} \\
         & =\expects{(U,y)\sim \cD}{r_{\gamma}\left(-h(U).y\right) \middle |\,\left|h(U)\right|\geq \gamma }.  \prob{\left|h(U)\right|\geq \gamma } \\
         & + \expects{(U,y)\sim \cD}{r_{\gamma}\left(-h(U).y\right) \middle | \,\left|h(U)\right| < \gamma}. \prob{\left|h(U)\right|<\gamma },
\end{aligned} 
\end{equation}
where we used the fact that sigmoid is a monotonic increasing function with a unique inverse and that $\phi\left(c.\phi^{-1}(z)\right) = \phi\left(\phi^{-1}(\gamma)\right)=\gamma$. Notice that whenever $\left|h(U)\right|\geq \gamma$ we can also conclude that either $h(U).y\geq\gamma$ or $h(U).y\leq-\gamma$. This means that  $r_{\gamma}\left(-h(U).y\right)$ is either $0$ or $1$. When $h(U).y\geq\gamma$ we have $r_{\gamma}\left(-h(U).y\right)=0$ and when $h(U).y\leq-\gamma$ we have $r_{\gamma}\left(-h(U).y\right)=1$. In other words if $\left|h(U)\right|\geq \gamma$, we have
\begin{equation}\label{eq:sign_the_same}
    r_{\gamma}\left(-h(U).y\right) = \indicator{\sign{h(U)} \neq y}
\end{equation}
On the other hand, we know that $\gamma,z>0$ and $c = \phi^{-1}(\gamma)/\phi^{-1}(z)>0$. Consequently, $\sign{h(U)}=\sign{ \last{\hat{b}^R\left(U,T-1\right)}} = f(U)$ for any $U\in\bR^{p\times T}$. Lemma~\ref{lemma:last_similar} suggests that
\begin{equation*}
    \prob{\left|\last{b^R(U,T-1)}\right|<z} = \prob{\left|\last{\hat{b}^R(U,T-1)}\right|<\phi\left(c.\phi^{-1}(z)\right)} = \prob{\left|h(U)\right|< \gamma}.
\end{equation*}
Moreover, we know that $z$ is chosen such that $\prob{\left|\last{b^R(U,T-1)}\right|<z} < \eta$ and the ramp loss is at most 1. Taking this and Equations~\ref{eq:prob_ramp_loss} and \ref{eq:sign_the_same} into account we can write that
\begin{equation*}
\begin{aligned}
         &\expects{(U,y)\sim \cD}{l_{\gamma}\left(h,U,y\right)}
         = \expects{(U,y)\sim \cD}{\indicator{\sign{h(U)} \neq y} \middle| \,\left|h(U)\right|\geq \gamma }. \prob{\left|h(U)\right|\geq \gamma  }\\
         & + \expects{(U,y)\sim \cD}{r_{\gamma}\left(-h(U).y\right) \middle | \left|h(U)\right| < \gamma}. \prob{\left|\last{b^R(U,T-1)}\right|<z}\\
         & \leq  \expects{(U,y)\sim \cD}{\indicator{\sign{h(U)} \neq y} \middle| \,\left|h(U)\right|\geq \gamma }. \prob{\left|h(U)\right|\geq \gamma  } +\eta\\
         & \leq  \expects{(U,y)\sim \cD}{\indicator{\sign{h(U)} \neq y} \middle| \,\left|h(U)\right|\geq \gamma }. \prob{\left|h(U)\right|\geq \gamma  }\\ 
         & +  \expects{(U,y)\sim \cD}{\indicator{\sign{h(U)} \neq y} \middle| \,\left|h(U)\right|< \gamma }. \prob{\left|h(U)\right|< \gamma  } +\eta\\
         & \leq  \expects{(U,y)\sim \cD}{\indicator{\sign{h(U)} \neq y} } +\eta\\
         & \leq \expects{(U,y)\sim \cD}{l^{0-1}\left(f,U,y\right)}+ \eta.
\end{aligned} 
\end{equation*}

\end{proof}

{\bf Proof of Theorem~\ref{thm:lower_rnn}.}\begin{proof}
    Define $\cF_w = \{f:[-1/2,1/2]\to \{-1,1\}\mid f(U) = \sign{h(U)},h\in\cH_w\}$ as the class of all sigmoid recurrent networks with $w$ weights that output binary values. Let $\cD$ be a distribution over $[-1/2,1/2]^{p \times T} \times \{-1,1\}$. From Lemma~\ref{lemma:ramp_to_zero_one} we know that for every $f\in\cF_w$ there exists a function $h\in\cH_w$ such that $\expects{(U,y)\sim \cD}{l_{\gamma}\left(h,U,y\right)} \leq \expects{(U,y)\sim \cD}{l^{0-1}\left(f,U,y\right)} + \eta$, where $\eta>0$ is any small value. Therefore, we can write that
\begin{equation}\label{eq:inf_ramp_zero_one}
    \inf_{h\in\cH_w}\expects{(U,y)\sim \cD}{l_{\gamma}\left(h,U,y\right)} \leq \inf_{f\in\cF_w}\expects{(U,y)\sim \cD}{l^{0-1}\left(f,U,y\right)} + \eta.
\end{equation}
Let $m_{\cH_w}(\epsilon,\delta)$ denote the sample complexity of PAC learning $\cH_w$ with respect to ramp loss. Therefore, there exists an algorithm $\cA$ that receives a set $S$ of $m\geq m_{\cH_w}(\epsilon,\delta)$ i.i.d. samples from $\cD$ and returns $\hat{h}=\cA(S)$ such that with probability at least $1-\delta$ we have
\begin{equation*}
   \expects{(U,y)\sim \cD}{l^{\gamma}\left(\hat{h},U,y\right)} \leq \inf_{h \in \cH_w}\expects{(U,y)\sim \cD}{l^{\gamma}\left(h,U,y\right)} + \epsilon.
\end{equation*}
Let $\hat{f}$ be a function in $\cF_w$ such that $\hat{f}(U) = \sign{\hat{h}(U)}$ for every $U\in[-1/2,1/2]^{p\times T}$. Given the definitions of $0-1$ loss and ramp loss, it is easy to verify that $  \expects{(U,y)\sim \cD}l^{0-1}(\hat{f},U,y) \leq \expects{(U,y)\sim \cD}l^{\gamma}(\hat{h},U,y)$. Taking this and Equation~\ref{eq:inf_ramp_zero_one} into account, we can define a new algorithm $\cA'$ that, given the set $S$, returns $\hat{f}\in \cF_w$ such that with probability at least $1-\delta$ we have
\begin{equation*}
   \expects{(U,y)\sim \cD}l^{0-1}(\hat{f},U,y) \leq \inf_{h \in \cH_w}\expects{(U,y)\sim \cD}l^{\gamma}(h,U,y)+ \epsilon \leq \inf_{f\in\cF_w}\expects{(U,y)\sim \cD}l^{0-1}(f,U,y) + \epsilon +\eta.
\end{equation*}
This means that we have
\begin{equation}\label{eq:sample_complexities}
    m_{\cF_w}^{0-1}(\epsilon+\eta,\delta) \leq m_{\cH_w}(\epsilon,\delta).
\end{equation}
On the other hand, from Theorem~\ref{thm:lower_vc} we now that the VC-dimension of $\cF_w$ is $\Omega(wT)$. Moreover, Theorem~\ref{thm:fundamental} suggests that
\begin{equation*}
    m_{\cF_w}^{0-1}(\epsilon,\delta) = \Omega\left(\frac{wT+\log(1/\delta)}{\epsilon^2}\right).
\end{equation*}
Taking the above equation and Equation~\ref{eq:sample_complexities} into account, by setting $\eta=O(\epsilon)$, we can write that
\begin{equation*}
    m_{\cH_w}(\epsilon,\delta) =  \Omega\left(\frac{wT+\log(1/\delta)}{\epsilon^2}\right),
\end{equation*}
which concludes our result.
\end{proof} 

The following theorem states that we can find a lower bound on the sample complexity of PAC learning $\cF$ with respect to $0-1$ loss based on its VC-dimension. For a proof see Theorems 5.2, and 5.10 in \citet{anthony1999neural}.
\begin{theorem}[Lower Bound on the Sample Complexity of PAC Learning \citep{anthony1999neural}]\label{thm:fundamental}
Let $\cF$ be a class of functions from a domain $\cX$ to $\{-1,1\}$ and let $d = \mbox{VC}(\cF)$ be the VC-dimension of the class $\cF$. Assume $d<\infty$. Then there exists an absolute constant $C$ such that for every $(\epsilon,\delta)\in(0,1/40)$ we have
\begin{equation*}
    m_{\cF}^{0-1}(\epsilon,\delta) \geq C\frac{d+\log(1/\delta)}{\epsilon^2}.
\end{equation*}
\end{theorem}

We now introduce a lower bound on the VC-dimension of sigmoid recurrent neural networks with binary outputs which is based on a result due to \citet{koiran1998vapnik}.
\begin{theorem}[A Lower Bound on VC-Dimension of Sigmoid Recurrent Neural Networks]\label{thm:lower_vc}
For every $T\geq 3$ and $w\geq 19$ there exists a well-defined class $\cH_w = \brmodel{\mnet{p_0}{p_k}{w}}{T}$ with the following property: The VC-dimension of $\cF_w = \{f:[-1/2,1/2]^{p\times T}\to \{-1,1
\}\mid f(U) = \sign{h(U)},h\in\cH_w\}$ is $\Omega\left(wT\right)$.
\end{theorem}
The proof of the above theorem is essentially the same as the proof of the result in \citet{koiran1998vapnik}. The only difference is that the we should construct our network in a way that the last two dimensions of the output of $\mnet{p_0}{p_k}{w}$ must be similar to each other in order to feed back the value of $\last{b^{R}(f,t-1)}$ with an extra node. Therefore, we only need a network that has a constant factor more weights than the network that is proposed in \citet{koiran1998vapnik} which does not change the order of sample complexity.
\subsection{Lemmas used in the proof of Lemma~\ref{lemma:ramp_to_zero_one}}
In the following, we state two lemmas that will help in proving Lemma~\ref{lemma:ramp_to_zero_one}. 
\begin{lemma}\label{lemma:first_similar}
Let $W_{k-1} = \begin{bmatrix} v_1 \ldots v_{p_k}\end{bmatrix}\in\bR^{p_{k-1}\times p_k}$ and $\hat{W}_{k-1} = \begin{bmatrix} v_1 & \ldots & c.v_{p_k}\end{bmatrix}\transpose$ for a constant $c>0$. Define two single-layer networks $b_{k-1}(x) = \Phi\left(W_{k-1}\transpose x\right)$ and $\hat{b}_{k-1}(x) = \Phi\left(\hat{W}_{k-1}\transpose x\right)$. Then, for any two multi-layer networks $b = b_{k-1}\circ\ldots \circ b_0$ and $\hat{b} = \hat{b}_{k-1}\circ\ldots \circ b_0$ in a well-defined class $\mnet{p_0}{p_k}{w}$, every $U\in[-1/2,1/2]^{p\times T}$, and every $t\in[T-1]$ we have
\begin{equation*}
 \first{b^R\left( U,t \right)} = \first{\hat{b}^R\left( U,t\right)}.
\end{equation*}
\end{lemma}
\begin{proof}
We prove by induction. Denote $r = b_{k-2}\circ\ldots\circ b_o$. Therefore, we have $b = b_{k-1} \circ r$ and $\hat{b} = \hat{b}_{k-1}\circ r$. For $t=0$, we can denote $x^{(0)} = r\left(\begin{bmatrix}
        \zero{q-1} &
        u^{(0)}
    \end{bmatrix}\transpose\right)$ and write that 
\begin{equation*}
\begin{aligned}
    &    \first{b^R\left(U,0\right)} = \first{b_{k-1}\left(r\left(\begin{bmatrix}
        \zero{q-1}\\
        u^{(0)}
    \end{bmatrix}\right)\right)} = \first{b_{k-1}\left(x^{(0)}\right)} \\ 
    &= \first{\begin{bmatrix}
        \phi\left(\langle v_1,x^{(0)} \rangle\right) \\
        \vdots \\
        \phi\left(\langle v_{p_{k}-1},x^{(0)} \rangle\right)
    \end{bmatrix}}  = \first{\begin{bmatrix}
        \phi\left(\langle v_1,x^{(0)} \rangle\right) \\
        \vdots \\
        \phi\left(\langle c.v_{p_{k}-1},x^{(0)} \rangle\right)
    \end{bmatrix}}  \\
    &=\first{\hat{b}_{k-1}\left(x^{(0)}\right)} =   \first{\hat{b}^R\left(U,0\right)},
\end{aligned}
\end{equation*}
where $\langle v_i,x^{(t)}\rangle$ denotes the inner product between vectors $v_i$ and $x^{(t)}$. Assume that we have $\first{b^R\left(U,t-1\right)} = \first{\hat{b}^R\left(U,t-1\right)}$ for $t-1\in[T-2]$. We now prove that we also have $\first{b^R\left(U,t\right)} = \first{\hat{b}^R\left(U,t\right)}$. Denote $x^{(t)} = r\left( \begin{bmatrix}
    \first{b^R\left(U,t-1\right)} & u^{(t)}
\end{bmatrix}\transpose\right)$. We can then write that 
\begin{equation*}
\begin{aligned}
    &    \first{b^R\left(U,t\right)} = \first{b_{k-1}\circ r\left(\begin{bmatrix}
        \first{b^R\left(U,t-1\right)} \\
        u^{(t)}
    \end{bmatrix}\right)} = \first{b_{k-1}\left(x^{(t)}\right)} \\
    &= \first{\begin{bmatrix}
        \phi\left(\langle v_1,x^{(t)} \rangle \right)\\
        \vdots \\
        \phi\left(\langle v_{p_{k}-1},x^{(t)} \rangle\right)
    \end{bmatrix}} = \first{\begin{bmatrix}
        \phi\left(\langle v_1,x^{(t)} \rangle \right)\\
        \vdots \\
        \phi\left(\langle c.v_{p_{k}-1},x^{(t)} \rangle\right)
    \end{bmatrix}} \\
   & = \first{\hat{b}_{k-1}\left(x^{(t)}\right)} =   \first{\hat{b}^R\left(U,t\right)}.
\end{aligned}
\end{equation*} 
\end{proof}

\begin{lemma}\label{lemma:last_similar}
Let $W_{k-1} = \begin{bmatrix} v_1 \ldots v_{p_k}\end{bmatrix}\in\bR^{p_{k-1}\times p_k}$ and $\hat{W}_{k-1} = \begin{bmatrix} v_1 & \ldots & c.v_{p_k}\end{bmatrix}\transpose$ for a constant $c>0$. Define two single-layer networks $b_{k-1}(x) = \Phi\left(W_{k-1}\transpose x\right)$ and $\hat{b}_{k-1}(x) = \Phi\left(\hat{W}_{k-1}\transpose x\right)$. Let $\cD$ be a distribution over $[-1/2,1/2]^{p\times T}$. Then, for any two multi-layer networks $b = b_{k-1}\circ\ldots \circ b_0$ and $\hat{b} = \hat{b}_{k-1}\circ\ldots \circ b_0$ in a well-defined class $\mnet{p_0}{p_k}{w}$ we have
\begin{equation*}
    \prob{\left|\last{b^R\left( U,T-1 \right)}\right| < z}   =\prob{\left|\last{\hat{b}^R\left( U,T-1 \right)}\right| < \phi\left(c.\phi^{-1}\left(z\right)\right)}, 
\end{equation*}
where $\phi^{-1}(z)$ is the inverse of sigmoid function $\phi$ at $z$. 
\end{lemma}
\begin{proof}
    Denote $r = b_{k-2}\circ\ldots\circ b_o$ and $x^{(T-1)} = r\left( \begin{bmatrix}
    \first{b^R\left(U,T-2\right)} & u^{(T-1)}
\end{bmatrix}\transpose\right)$. Note that 
\begin{equation*}
\begin{aligned}
     & \last{b^R\left( U,T-1 \right)} = \last{b_{k-1}\circ r\left(\begin{bmatrix}
        \first{b^R\left(U,T-2\right)} \\
        u^{(T-1)}
    \end{bmatrix}\right)} \\
    &= \last{b_{k-1}\left(x^{(T-1)}\right)} = 
        \phi\left(\langle v_{p_k},x^{(T-1)} \rangle \right),
\end{aligned}
\end{equation*}
where $\langle v_{p_k},x^{(T-1)} \rangle$ denotes the inner product between $v_{p_k}$ and $x^{(T-1)}$.
From Lemma~\ref{lemma:first_similar}, we know that  $\first{b^R\left(U,T-2\right)} = \first{\hat{b}^R\left(U,T-2\right)}$. Therefore, we also have that
\begin{equation*}
\begin{aligned}
     & \last{\hat{b}^R\left( U,T-1 \right)} = \last{\hat{b}_{k-1}\circ r\left(\begin{bmatrix}
        \first{\hat{b}^R\left(U,T-2\right)} \\
        u^{(T-1)}
    \end{bmatrix}\right)} \\
    &= \last{\hat{b}_{k-1}\left(x^{(T-1)}\right)} = 
        \phi\left(\langle c.v_{p_k},x^{(T-1)} \rangle \right).
\end{aligned}
\end{equation*}
Considering the above equations and the facts that $\phi(x)$ is an invertible and strictly increasing function and that $\phi(x) = -\phi(-x)$, we can write
\begin{equation*}
    \begin{aligned}\label{eq:similar_probs}
       & \prob{\left|\last{b^R\left( U,T-1 \right)}\right| < z} = \prob{-z < \last{b^R\left( U,T-1 \right)} < z}\\
      & = \prob{-z \leq \phi\left(\langle v_{p_k},x^{(T-1)} \rangle \right) < z}  = \prob{\phi^{-1}(-z) < \langle v_{p_k},x^{(T-1)} \rangle < \phi^{-1}\left(z\right)} \\
      & = \prob{ -c.\phi^{-1}(z)\leq \langle c.v_{p_k},x^{(T-1)} \rangle < c.\phi^{-1}\left(z\right)}  \\
      & =\prob{ \phi\left(-c.\phi^{-1}\left(z\right)\right) < \phi\left(\langle c.v_{p_k},x^{(T-1)} \rangle\right) < \phi\left(c.\phi^{-1}\left(z\right)\right)}\\
        & =\prob{ -\phi\left(c.\phi^{-1}\left(z\right)\right) < \phi\left(\langle c.v_{p_k},x^{(T-1)} \rangle\right) < \phi\left(c.\phi^{-1}\left(z\right)\right)}\\
        &=\prob{\left|\last{\hat{b}^R\left( U,T-1 \right)}\right|< \phi\left(c.\phi^{-1}\left(z\right)\right)}.
    \end{aligned}
\end{equation*}
\end{proof}

\section{Proof of upper bound}\label{app:upper_bound}

\subsection{{Proof of Theorem~\ref{thm:cover_recurrent_model}}}
We prove the following general theorem which holds for input domains $\rv{\cX_s}$ and $\Delta_{p\times T}$.

\begin{theorem}[TV Covering Number of $\rv{\cG_{\sigma}}\circ\brmodel{\rv{\cF_{\sigma}}}{T}$ From $\rv{\cG_{\sigma}}\circ\rv{\cF_\sigma}$]
Let $s,p,q\in \bN$ such that $s=p+q-1$.
Let $\rv{\cF}$ be a class of functions from $\rv{\cX_{s}}$ to $\rv{\cX_{q}}$ and denote by $\rv{\cF_{\sigma}} = \rv{\cF}\circ\rv{\cG_{\sigma,s}}$ the class of its composition with noise. Then we have

\begin{equation*}
N_U\left(\epsilon,\rv{\cG_\sigma}\circ\brmodel{\rv{\cF_{\sigma}}}{T},\infty,d_{TV}^{\infty},\rv{\Delta_{p\times T}}\right) \leq N_U\left(\epsilon/T,\rv{\cG_{\sigma,q}}\circ\rv{\cF_{\sigma}},\infty,d_{TV}^{\infty},\rv{\cX_{s}}\right).
\end{equation*}
\end{theorem}
\begin{proof}
    Let $C=\{\rv{g_{\sigma,q}}\circ\hat{f}_{i}\circ \rv{g_{\sigma,s}}\mid \hat{f}_i\circ \rv{g_{\sigma,s}} \in \rv{\cF_{\sigma}}, i\in[r]\}$ be a global $\epsilon$-cover for $\rv{\cG_{\sigma,s}}\circ\rv{\cF_{\sigma}}$ with respect to domain $\rv{\cX_{s}}$ and $d_{TV}^{\infty}$. Therefore, $|C|\leq N_U\left(\epsilon,\rv{\cG_{\sigma,q}}\circ\rv{\cF_{\sigma}},\infty,d_{TV}^{\infty},\rv{\cX_{s}}\right)$. Then for any function $\rv{g_{\sigma,q}}\circ f\circ\rv{g_{\sigma,s}}\in\rv{\cG_{\sigma,q}}\circ\rv{\cF_{\sigma}}$ we know that there exists a function $\rv{g_{\sigma,q}}\circ\hat{f}_i\circ\rv{g_{\sigma,s}}$ in $C$ such that for every $\rv{x}\in\rv{\cX_{s}}$ we have $d_{TV}\left(\rv{g_{\sigma,q}}\circ f\circ\rv{g_{\sigma,s}}(\rv{x}),\rv{g_{\sigma,q}}\circ\hat{f}_i\circ\rv{g_{\sigma,s}}(\rv{x})\right)\leq \epsilon$.
   Denote $
   \rv{h} = f\circ\rv{g_{\sigma,s}}$ and $\rv{\hat{h}_i} = \hat{f}_i\circ\rv{g_{\sigma,s}}$. We prove by induction that for any input matrix $\rv{U} = \begin{bmatrix}
        \rv{u^{(0)}} &\ldots &\rv{u^{(T-1)}}
    \end{bmatrix}\in\rv{\Delta_{p\times T}}$, where $\rv{u^{(t)}} = \rv{\delta_{u^{(t)}}}$, we have $d_{TV}\left(\rv{g_{\sigma,q}}\circ \rv{h}^R\left(\rv{U},T-1\right), \rv{g_{\sigma,q}}\circ \rv{\hat{h}_i}^R\left(\rv{U},T-1\right)\right)\leq T\epsilon$. 
    
    We start by proving that $d_{TV}\left(\rv{g_{\sigma,q}}\circ \rv{h}^R\left(\rv{U},0\right),\rv{g_{\sigma,q}}\circ \rv{\hat{h}_i}^R\left(\rv{U},0\right)\right)\leq \epsilon$. Denote $\rv{x^{(0)}} = \begin{bmatrix}\rv{\delta_{\zero{q-1}}} & \rv{u^{(0)}}\end{bmatrix}\transpose\in\rv{\Delta_{s}}$. We can write that
    \begin{equation*}
        \begin{aligned}
            &d_{TV}\left(\rv{g_{\sigma,q}}\circ h^R\left(\rv{U},0\right),\rv{g_{\sigma,q}}\circ \hat{h}_i^R\left(\rv{U},0\right)\right) \\
            &= d_{TV}\left(\rv{g_{\sigma,q}}\circ f\circ\rv{g_{\sigma,s}}\left(\begin{bmatrix}\rv{\delta_{\zero{q-1}}} \\ \rv{u^{(0)}}\end{bmatrix}\right),\rv{g_{\sigma,q}}\circ \hat{f}_i\circ\rv{g_{\sigma,s}}\left(\begin{bmatrix}\rv{\delta_{\zero{q-1}}} \\ \rv{u^{(0)}}\end{bmatrix}\right)\right).
        \end{aligned}
    \end{equation*}
    Since $\left(\begin{bmatrix}\rv{\delta_{\zero{q-1}}} & \rv{u^{(0)}}\end{bmatrix}\transpose\right)\in \rv{\cX_{s}}$ and considering the fact that $\rv{g_{\sigma,q}}\circ f\circ\rv{g_{\sigma,s}}= \rv{g_{\sigma,q
    }}\circ h \in\rv{\cG_{\sigma,q}}\circ\rv{\cF_{\sigma}}$ and $\rv{g_{\sigma,q}}\circ \hat{f}_i\circ\rv{g_{\sigma,s}}= \rv{g_{\sigma,q
    }}\circ \rv{\hat{h}_i} \in\rv{\cG_{\sigma,q}}\circ\rv{\cF_{\sigma}}$ are globally $\epsilon$-close over $\rv{\cX_{s}}$, we get that
    \begin{equation*}
        d_{TV}\left(\rv{g_{\sigma,q}}\circ \rv{h}^R\left(\rv{U},0\right),\rv{g_{\sigma,q}}\circ \rv{\hat{h}_i}^R\left(\rv{U},0\right)\right) \leq \epsilon.
    \end{equation*}
Now assume that we have
\begin{equation}\label{eq:induc}
     d_{TV}\left(\rv{g_{\sigma,q}}\circ \rv{h}^R\left(\rv{U},t-1\right), \rv{g_{\sigma,q}}\circ \rv{\hat{h}_i}^R\left(\rv{U},t-1\right) \right)
    \leq t\epsilon.
\end{equation}
We want to bound the total variation distance between $\rv{g_{\sigma,q}}\circ \rv{h}^R\left(\rv{U},t\right)$ and $\rv{g_{\sigma,q}}\circ \rv{\hat{h}_i}^R\left(\rv{U},t\right)$, which are defined as follows. 
\begin{equation*}
\begin{aligned}
      &\rv{g_{\sigma,q}}\circ \rv{h}^R\left(\rv{U},t\right) =  \rv{g_{\sigma,q}}\circ f\circ\rv{g_{\sigma,s}}\left(\begin{bmatrix}
        \first{\rv{h}^R\left(\rv{U},t-1\right)}\\
        \rv{u^{(t)}}
    \end{bmatrix}\right),\\
     &\rv{g_{\sigma,q}}\circ \rv{\hat{h}_i}^R\left(\rv{U},t\right) =  \rv{g_{\sigma,q}}\circ \hat{f}_i\circ\rv{g_{\sigma,s}}\left(\begin{bmatrix}
        \first{\rv{\hat{h}_i}^R\left(\rv{U},t-1\right)}\\
        \rv{u^{(t)}}
    \end{bmatrix}\right).
\end{aligned}
\end{equation*}
 From Lemma~\ref{lemma:tv_to_tv_first} we know that
 \begin{equation*}
       \begin{aligned}
           &d_{TV}\left(\first{\rv{g_{\sigma,q}}\left( \rv{h}^R\left(\rv{U},t-1\right)\right)}, \first{\rv{g_{\sigma,q}}\left( \rv{\hat{h}_i}^R\left(\rv{U},t-1\right) \right)}\right)\\
        &   \leq d_{TV}\left(\rv{g_{\sigma,q}}\left( \rv{h}^R\left(\rv{U},t-1\right)\right), \rv{g_{\sigma,q}}\left( \rv{\hat{h}_i}^R\left(\rv{U},t-1\right) \right)\right)
    \leq t\epsilon
       \end{aligned} 
 \end{equation*}
It is easy to verify that $\first{\rv{g_{\sigma,q}}\left( \rv{h}^R\left(\rv{U},t-1\right)\right)} = \rv{g_{\sigma,q-1}}\left(\first{\rv{h}^R\left(\rv{U},t-1\right)}\right)$ because $\rv{g_{\sigma,q}}$ is a gaussian noise with covariance matrix equal to $\sigma^2I_{q}$, where $I_q\in\bR^{q\times q}$ is the identity matrix. Considering this fact and  Lemma~\ref{lemma:tv_to_tv_concatenate} we can write that 
\begin{equation*}
\begin{aligned}
     &  d_{TV}\left(\begin{bmatrix}
        \rv{g_{\sigma,q-1}}\left(\first{\rv{h}^R\left(\rv{U},t-1\right)}\right)\\
        \rv{u^{(t)}}
    \end{bmatrix},  \begin{bmatrix}
        \rv{g_{\sigma,q-1}}\left(\first{\rv{\hat{h}_i}^R\left(\rv{U},t-1\right)}\right)\\
        \rv{u^{(t)}}
    \end{bmatrix}\right)\\
   \leq & d_{TV}\left(\first{\rv{g_{\sigma,q}}\left( \rv{h}^R\left(\rv{U},t-1\right)\right)}, \first{\rv{g_{\sigma,q}}\left( \rv{\hat{h}_i}^R\left(\rv{U},t-1\right) \right)}\right) \leq t\epsilon.
\end{aligned}
\end{equation*}
Applying data processing inequality for TV distance (i.e., Lemma~\ref{lemma:DPI}) we can write that 
\begin{equation}\label{eq:base_induc}
\begin{aligned}
      & d_{TV}\left(\begin{bmatrix}
        \rv{g_{\sigma,q-1}}\left(\first{\rv{h}^R\left(\rv{U},t-1\right)}\right)\\
        \rv{g_{\sigma,p}}\left(\rv{u^{(t)}}\right)
    \end{bmatrix},  \begin{bmatrix}
        \rv{g_{\sigma,q-1}}\left(\first{\rv{\hat{h}_i}^R\left(\rv{U},t-1\right)}\right)\\
        \rv{g_{\sigma,p}}\left(\rv{u^{(t)}}\right)
    \end{bmatrix}\right)\\
       & = d_{TV}\left(\rv{g_{\sigma,s}}\left(\begin{bmatrix}
       \first{\rv{h}^R\left(\rv{U},t-1\right)}\\
        \rv{u^{(t)}}
    \end{bmatrix}\right), \rv{g_{\sigma,s}}\left( \begin{bmatrix}
        \first{\rv{\hat{h}_i}^R\left(\rv{U},t-1\right)}\\
       \rv{u^{(t)}}
    \end{bmatrix}\right)\right)\leq t\epsilon.
\end{aligned}
\end{equation}

Notice that $\begin{bmatrix}
        \first{\rv{h}^R\left(\rv{U},t-1\right)}&
        \rv{u^{(t)}}
    \end{bmatrix}\transpose$ is in $\rv{\cX_{s}}$. Since we know that  $\rv{g_{\sigma,q}}\circ f\circ\rv{g_{\sigma,s}}$ and $\rv{g_{\sigma,q}}\circ \hat{f}_i\circ\rv{g_{\sigma,s}}$ are globally $\epsilon$-close on $\rv{\cX_{s}}$, we can write that
\begin{equation}\label{eq:triangle1}
\begin{aligned}
      d_{TV}\left(\rv{g_{\sigma,q}}\circ f\circ\rv{g_{\sigma,s}}\left(\begin{bmatrix}
       \first{\rv{h}^R\left(\rv{U},t-1\right)}\\
        \rv{u^{(t)}}
    \end{bmatrix}\right), \rv{g_{\sigma,q}}\circ\hat{f}_i\circ\rv{g_{\sigma,s}}\left( \begin{bmatrix}
        \first{\rv{h}^R\left(\rv{U},t-1\right)}\\
       \rv{u^{(t)}}
\end{bmatrix}\right)\right)\leq \epsilon.
\end{aligned}
\end{equation}
Moreover, from data processing inequality (i.e., Lemma~\ref{lemma:DPI}) and Equation~\ref{eq:base_induc} we can conclude that 
    \begin{equation}\label{eq:triangle2}
        d_{TV}\left( \rv{g_{\sigma,q}}\circ \rv{\hat{f}_i}\circ\rv{g_{\sigma,s}}\left(\begin{bmatrix}
        \first{\rv{h}^R\left(\rv{U},t-1\right)}\\
        \rv{u^{(t)}}
    \end{bmatrix}\right),
     \rv{g_{\sigma,q}}\circ \hat{f}_i\circ\rv{g_{\sigma,s}}\left(\begin{bmatrix}
        \first{\rv{\hat{h}_i}^R\left(\rv{U},t-1\right)}\\
        \rv{u^{(t)}}
    \end{bmatrix}\right)\right)\leq t\epsilon.
    \end{equation}

Finally, we can combine Equations~\ref{eq:triangle1} and \ref{eq:triangle2} together with the triangle inequality for total variation distance to conclude that
    \begin{equation*}
    \begin{aligned}
    &    d_{TV}\left( \rv{g_{\sigma,q}}\circ f\circ\rv{g_{\sigma,s}}\left(\begin{bmatrix}
        \first{\rv{h}^R\left(\rv{U},t-1\right)}\\
        \rv{u^{(t)}}
    \end{bmatrix}\right),
     \rv{g_{\sigma,q}}\circ \hat{f}_i\circ\rv{g_{\sigma,s}}\left(\begin{bmatrix}
        \first{\rv{\hat{h}_i}^R\left(\rv{U},t-1\right)}\\
        \rv{u^{(t)}}
    \end{bmatrix}\right)\right)\\
    &= d_{TV}\left(\rv{g_{\sigma,q}}\circ \rv{h}^R\left(\rv{U},t\right),\rv{g_{\sigma,q}}\circ \rv{\hat{h}_i}^R\left(\rv{U},t\right)\right)\leq (t+1)\epsilon.
    \end{aligned}  
    \end{equation*}
    So far, we have proved that for any input matrix $\rv{U}\in\rv{\Delta_{p\times T}}$ we have 
    \begin{equation*}
        d_{TV}\left(\rv{g_{\sigma,q}}\circ \rv{h}^R\left(\rv{U},T-1\right),\rv{g_{\sigma,q}}\circ \rv{\hat{h}_i}^R\left(\rv{U},T-1\right)\right)\leq T\epsilon
    \end{equation*}
    By another application of Lemma~\ref{lemma:tv_to_tv_first} we can conclude that 
      \begin{equation*}
        d_{TV}\left(\last{\rv{g_{\sigma,q}}\circ \rv{h}^R\left(\rv{U},T-1\right)},\last{\rv{g_{\sigma,q}}\circ \rv{\hat{h}_i}^R\left(\rv{U},T-1\right)}\right)\leq T\epsilon
    \end{equation*}
    We can have a similar argument to the first function and write the above equation as
        \begin{equation*}
        d_{TV}\left(\rv{g_{\sigma,1}}\circ\last{h^R\left(\rv{U},T-1\right)},\rv{g_{\sigma,1}}\circ\last{\hat{h}_i^R\left(\rv{U},T-1\right)}\right)\leq T\epsilon.
    \end{equation*}
    This means that for every function $\rv{g_{\sigma,1}}\circ\last{\rv{h}^R\left(\rv{U},T-1\right)}$ in $\rv{\cG_{\sigma,1}}\circ\brmodel{\rv{\cF_{\sigma}}}{T}$ there exists a function $\hat{f}_i$ in $\cF$ such that $\rv{g_{\sigma,1}}\circ\last{\rv{h}^R\left(\rv{U},T-1\right)}$ and $\rv{g_{\sigma,1}}\circ\last{\rv{\hat{h}_i}^R\left(\rv{U},T-1\right)}$ are globally $T\epsilon$-cover close to each other with respect to $\rv{\Delta_{p\times T}}$.
    Setting $\epsilon'=\epsilon/T$ we can conclude that
    \begin{equation*}
N_U\left(\epsilon,\rv{\cG_\sigma}\circ\brmodel{\rv{\cF_{\sigma}}}{T},\infty,d_{TV}^{\infty},\rv{\Delta_{p\times T}}\right) \leq N_U\left(\frac{\epsilon}{T},\rv{\cG_{\sigma,q}}\circ\rv{\cF_{\sigma}},\infty,d_{TV}^{\infty},\rv{\cX_{s}}\right).
\end{equation*}

    The proof of the bounded domains essentially follows the same steps as above but for inputs that are bounded, i.e., inputs in $\rv{\Delta_{B,p\times T}}$ and $\rv{\cX_{B,s}}$.
\end{proof}

\subsection{A bound on the TV covering number of multi-layer noisy networks}
From Theorem~\ref{thm:cover_single} and Theorem~\ref{thm:compose_tv} we can get the following bound on the total variation covering number of noisy multi-layer networks.
\begin{theorem}[TV Cover for Multi-Layer Noisy Neural Networks]\label{thm:cover_multi} 
For every $\epsilon,\sigma\in(0,1)$ and every well-defined class $\rvmnet{p_0}{p_k}{w}{\sigma}$, we have
\begin{align*}
   \log &N_U(\epsilon,\rv{\cG_{\sigma,p_k}}\circ\rvmnet{p_0}{p_k}{w}{\sigma},\infty,d_{TV}^{\infty},\rv{\cX_{0.5,p_0}}) \\
   &= O\left( w\log\left(\frac{w}{\epsilon\sigma}\log\left(\frac{w}{\epsilon\sigma}\right)\right)\right) = \widetilde{O}\left(w\log\left(\frac{1}{\epsilon\sigma}\right)\right),
\end{align*}
where $\widetilde{O}$ hides logarithmic factors.
\end{theorem}
\begin{proof}
    Fix a choice of $p_1,\ldots,p_{k-1}\in\bN$ and let $\rv{\cF}=\net{p_{k-1}}{p_k}\circ\ldots\circ\rv{\cG_{\sigma}}\circ\net{p_0}{p_1}\circ\rv{\cG_{\sigma}}$ be a class of multi-layer sigmoid neural networks in $\rvmnet{p_0}{p_k}{w}{\sigma}$. Notice that 
    \begin{equation*}
      \rv{\cG_{\sigma}} \circ \rv{\cF} = \rv{\cG_{\sigma}}\circ\net{p_{k-1}}{p_k}\circ\ldots\circ\rv{\cG_{\sigma}}\circ\net{p_0}{p_1}\circ\rv{\cG_{\sigma}}
    \end{equation*}
    and that the covering number of $\rv{\cG_{\sigma}} \circ \rv{\cF}$ with respect to $\rv{\cX_{1,p_0}}$ is the same as the covering number of $\rv{\cG_{\sigma}}\circ\net{p_{k-1}}{p_k}\circ\ldots\circ\rv{\cG_{\sigma}}\circ\net{p_0}{p_1}$ with respect to $\rv{\cG_{\sigma,p_0}}\circ\rv{\cX_{1,p_0}}$. From Theorem~\ref{thm:cover_single} we know that for any $0\leq i\leq k-1$ we can bound the covering number of $\rv{\cG_{\sigma}}\circ\net{p_i}{p_{i+1}}$ as \begin{equation*}
    \begin{aligned}
                &\log N_U\left(\epsilon,\rv{\cG_{\sigma}}\circ\net{p_i}{p_{i+1}},\infty,d_{TV}^{\infty},\rv{\cG_{\sigma}}\circ\rv{\cX_{0.5,p_i}}\right)\\
                &\leq p_i(p_{i+1}+1)\log\left(30\frac{p_i^{5/2}\sqrt{\ln\left((5p_i-\epsilon\sigma)/(\epsilon\sigma)\right)}}{\epsilon^{3/2}\sigma^2}\ln\left(\frac{5p_i}{\epsilon\sigma}\right)\right).
    \end{aligned}
    \end{equation*}
    Note that in \citet{pour2022benefits} the above bound was originally stated as a bound on the covering number of $\rv{\cG_{\sigma}}\circ\net{p_i}{p_{i+1}}$ with respect $\rv{\cG_{\sigma}}\circ\rv{\cX_{1,p_i}}$. However, we know that the bound with respect to $\rv{\cG_{\sigma}}\circ\rv{\cX_{1,p_i}}$ is always an upper bound for the covering number with respect to $\rv{\cG_{\sigma}}\circ\rv{\cX_{0.5,p_i}}$. If, instead of setting $B=1$, we wanted to consider $B=0.5$ as a bound on the domain, the covering number bound would become only tighter in terms of constant factors. Considering the above facts and applying Theorem~\ref{thm:compose_tv} recursively, we can write that
    \begin{equation*}
        \begin{aligned}
           &\log N_U\left(k\epsilon,\rv{\cG_{\sigma}} \circ \rv{\cF},\infty,d_{TV}^{\infty},\rv{\cX_{0.5,p_0}}\right)\\
                 &\leq \sum_{i=0}^{k-1 }\log N_U\left(\epsilon,\rv{\cG_{\sigma}}\circ\net{p_i}{p_{i+1}},\infty,d_{TV}^{\infty},\rv{\cG_{\sigma}}\circ\rv{\cX_{0.5,p_i}}\right)\\
                 &\leq \sum_{i=0}^{k-1 }p_i(p_{i+1}+1)\log\left(30\frac{p_i^{5/2}\sqrt{\ln\left((5p_i-\epsilon\sigma)/(\epsilon\sigma)\right)}}{\epsilon^{3/2}\sigma^2}\ln\left(\frac{5p_i}{\epsilon\sigma}\right)\right).  
        \end{aligned}
    \end{equation*}
    We can now set $\epsilon'=\epsilon/k$ and rewrite the above equation as
    \begin{equation*}
        \begin{aligned}
             &\log N_U\left(\epsilon,\rv{\cG_{\sigma}} \circ \rv{\cF},\infty,d_{TV}^{\infty},\rv{\cX_{0.5,p_0}}\right)\\
                &\leq \sum_{i=0}^{k-1 }p_i(p_{i+1}+1)\max_{i}\left\{\log\left(30\frac{p_i^{5/2}\sqrt{\ln\left((5p_i-\epsilon'\sigma)/(\epsilon'\sigma)\right)}}{\epsilon'^{3/2}\sigma^2}\ln\left(\frac{5p_i}{\epsilon'\sigma}\right)\right)\right\}\\
                &\leq w\max_{i}\left\{\log\left(30\frac{p_i^{5/2}\sqrt{\ln\left((5p_i-\epsilon'\sigma)/(\epsilon'\sigma)\right)}}{\epsilon'^{3/2}\sigma^2}\ln\left(\frac{5p_i}{\epsilon'\sigma}\right)\right)\right\}\\
                 &\leq w\max_{i}\left\{\log\left(30\frac{p_i^{5/2}\sqrt{\ln\left(5p_i/(\epsilon'\sigma)\right)}}{\epsilon'^{3/2}\sigma^2}\ln\left(\frac{5p_i}{\epsilon'\sigma}\right)\right)\right\}\\
                 &\leq w\max_{i}\left\{\log\left(30\frac{p_i^{5/2}\sqrt{5p_i/(\epsilon'\sigma)}}{\epsilon'^{3/2}\sigma^2}\ln\left(\frac{5p_i}{\epsilon'\sigma}\right)\right)\right\}\\
                 &\leq w\max_{i}\left\{\log\left(30\sqrt{5}\frac{p_i^{3}}{\epsilon'^{2}\sigma^{3/2}}\ln\left(\frac{5p_i}{\epsilon'\sigma}\right)\right)\right\}.
        \end{aligned}
    \end{equation*}
    Using the fact that $\epsilon,\sigma<1$, we can simplify the above equation and write that
    \begin{equation*}
        \begin{aligned}
         &\log N_U\left(\epsilon,\rv{\cG_{\sigma}} \circ \rv{\cF},\infty,d_{TV}^{\infty},\rv{\cX_{0.5,p_0}}\right)\\
                &\leq w\max_{i}\left\{\log\left((30\sqrt{5})^3\frac{p_i^{3}}{\epsilon'^{3}\sigma^{3}}\left(\ln\left(\frac{5p_i}{\epsilon'\sigma}\right)\right)^3\right)\right\}\\
               &\leq w\max_{i}\left\{3\log\left(30\sqrt{5}\frac{p_i}{\epsilon'\sigma}\ln\left(\frac{5p_i}{\epsilon'\sigma}\right)\right)\right\}\\
                &\leq w\max_{i}\left\{3\log\left(30\sqrt{5}\frac{kp_i}{\epsilon\sigma}\ln\left(\frac{5kp_i}{\epsilon\sigma}\right)\right)\right\}\\
                &\leq w\left(3\log\left(30\sqrt{5}\frac{w^2}{\epsilon\sigma}\ln\left(\frac{5w^2}{\epsilon\sigma}\right)\right)\right)\\
                &= O\left(w\log\left(\frac{w}{\epsilon\sigma}\ln\left(\frac{w}{\epsilon\sigma}\right)\right)\right)= \widetilde{O}\left(w\log\left(\frac{1}{\epsilon\sigma}\right)\right),
        \end{aligned}
    \end{equation*}
    where we used the fact that $k\leq w$ and $p_i\leq w$ for every $0\leq i\leq k$. Now that we found an upper bound on the covering number of $\rv{\cG_{\sigma}}\circ\rv{\cF}$ for a choice of $p_1,\ldots,p_{k-1}$, we can bound the covering number of $\rv{\cG_{\sigma}}\circ\rvmnet{p_0}{p_k}{w}{\sigma}$. The number of different choices that we can have for $p_1,\ldots,p_{k-1}$ is at most $w^{k-1}$ since we know that $\sum_{i=1}^k p_ip_{i-1}=w$ and therefore $p_i<w$ for every $0\leq i\leq k$. Therefore, we can simply take a union of the covering sets for each choice of $p_0,\ldots,p_{k-1}$ as a covering set for $\rv{\cG_{\sigma}}\circ\rvmnet{p_0}{p_k}{w}{\sigma}$, which yields to the following covering number bound.
    \begin{equation*}
        \begin{aligned}
         &\log N_U\left(\epsilon,\rv{\cG_{\sigma}}\circ\rvmnet{p_0}{p_k}{w} {\sigma},\infty,d_{TV}^{\infty},\rv{\cX_{0.5,p_0}}\right)\\
         &\leq \log w^{k}.N_U\left(\epsilon,\rv{\cG_{\sigma}}\circ\rv{\cF},\infty,d_{TV}^{\infty},\rv{\cG_{\sigma,p_0}}\circ\rv{\cX_{0.5,p_0}}\right)\\
         &\leq w\log w+\log N_U\left(\epsilon,\rv{\cG_{\sigma}}\circ\rv{\cF},\infty,d_{TV}^{\infty},\rv{\cG_{\sigma,p_0}}\circ\rv{\cX_{0.5,p_0}}\right) &&(k\leq w)\\
            &= O\left(w\log w + w\log\left(\frac{w}{\epsilon\sigma}\ln\left(\frac{w}{\epsilon\sigma}\right)\right)\right)= O\left(w\log\left(\frac{w}{\epsilon\sigma}\ln\left(\frac{w}{\epsilon\sigma}\right)\right)\right) = \widetilde{O}\left(w\log\left(\frac{1}{\epsilon\sigma}\right)\right),
        \end{aligned}
    \end{equation*}
\end{proof}
\subsection{Proof of Theorem~\ref{thm:cover_noisy_rnn}}
\begin{proof}
We know that
\begin{equation*}
      \rvmnet{p_0}{p_k}{w}{\sigma} = \bigcup \net{p_{k-1}}{p_k}    \circ \ldots \circ \rv{\cG_{\sigma}}\circ\net{p_1}{p_2}     \circ    \rv{\cG_{\sigma}}\circ\net{p_0}{p_1}\circ\rv{\cG_{\sigma}}.
\end{equation*}
Define $\rv{\cF} = \bigcup \net{p_{k-1}}{p_k}    \circ \ldots \circ \rv{\cG_{\sigma}}\circ\net{p_1}{p_2}     \circ    \rv{\cG_{\sigma}}\circ\net{p_0}{p_1}$ and note that $ \rv{\cF}\circ\rv{\cG_{\sigma}} = \rv{\cF_{\sigma}} = \rvmnet{p_0}{p_k}{w}{\sigma}$. Therefore, we can use Theorem~\ref{thm:cover_recurrent_model} to write that
    \begin{equation*}
        \begin{aligned}  &N_U\left(\epsilon,\rv{\cG_\sigma}\circ\brmodel{\rvmnet{p_0}{p_k}{w} {\sigma}}{T},\infty,d_{TV}^{\infty},\rv{\Delta_{0.5,p\times T}}\right)\\    &=N_U\left(\epsilon,\rv{\cG_\sigma}\circ\brmodel{\rv{\cF_{\sigma}}}{T},\infty,d_{TV}^{\infty},\rv{\Delta_{0.5,p\times T}}\right)\\
    & \leq N_U\left(\frac{\epsilon}{T},\rv{\cG_{\sigma}}\circ\rv{\cF_{\sigma}},\infty,d_{TV}^{\infty},\rv{\cX_{0.5,s}}\right)\\
    &=N_U\left(\frac{\epsilon}{T},\rv{\cG}\circ\rvmnet{p_0}{p_k}{w} {\sigma},\infty,d_{TV}^{\infty},\rv{\cX_{0.5,s}}\right).
        \end{aligned}
    \end{equation*}
    We know of a bound on the covering number of $\rv{\cG_{\sigma}}\circ\rvmnet{p_0}{p_k}{w} {\sigma}$ from Theorem~\ref{thm:cover_multi}. Using this bound we can rewrite the above equation as
    \begin{equation*}
        \begin{aligned}  &N_U\left(\epsilon,\rv{\cG_\sigma}\circ\brmodel{\rvmnet{p_0}{p_k}{w} {\sigma}}{T},\infty,d_{TV}^{\infty},\rv{\Delta_{0.5,p\times T}}\right)\\
    & \leq N_U\left(\frac{\epsilon}{T},\rv{\cG}\circ\rvmnet{p_0}{p_k}{w} {\sigma},\infty,d_{TV}^{\infty},\rv{\cX_{0.5,s}}\right)\\
    & =O\left(w\log\left(\frac{wT}{\epsilon\sigma}\ln\left(\frac{wT}{\epsilon\sigma}\right)\right)\right) = \widetilde{O}\left(w\log\left(\frac{T}{\epsilon\sigma}\right)\right).
        \end{aligned}
    \end{equation*}
\end{proof}

\subsection{Proof of Theorem~\ref{thm:pac_cover}}
\begin{proof}
    From Theorem~\ref{thm:pac_cover_exact} we can write that
  \begin{equation*}
\begin{aligned}
       &\expects{(x,y)\sim \cD}{l_{\gamma}(\hat{f},x,y)} \\
       &\leq  \inf_{f\in\cF} \expects{(x,y)\sim \cD}{l_{\gamma}(f,x,y)} +2\inf_{\epsilon\in [0,1/2]}\left\{2\left[4\epsilon + \frac{12}{\sqrt{m}}\int_{\epsilon}^{1/2}\sqrt{\ln N_U(\gamma\nu,\cF,m,\|.\|_2^{\ell_2})}\,d\nu\right]\right\}+ 6 \sqrt{\frac{\ln(2/\delta)}{2m}}\\
       & \leq \inf_{f\in\cF} \expects{(x,y)\sim \cD}{l_{\gamma}(f,x,y)} +  2\left[8\epsilon + \frac{24}{\sqrt{m}}\int_{\epsilon}^{1/2}\sqrt{\ln N_U(\gamma\nu,\cF,m,\|.\|_2^{\ell_2})}\,d\nu\right] + 6 \sqrt{\frac{\ln(2/\delta)}{2m}}\quad\quad\quad (\forall \epsilon \in [0,1/2])\\
     &\leq\inf_{f\in\cF} \expects{(x,y)\sim \cD}{l_{\gamma}(f,x,y)} +  16\epsilon + \frac{24}{\sqrt{m}}\sqrt{\ln N_U(\gamma\epsilon,\cF,m,\|.\|_2^{\ell_2})}+ 6 \sqrt{\frac{\ln(2/\delta)}{2m}},
\end{aligned}
\end{equation*}
where we have used the fact that the integral is over $[0,1/2]$ and the covering number decreases monotonically with $\epsilon$.
\end{proof}

\subsection{Proof of Theorem~\ref{thm:upper_rnn}}
We are now ready to state the proof of the upper bound on the sample complexity of PAC learning noisy recurrent neural networks with respect to the ramp loss.
\begin{proof}
    From Theorem~\ref{thm:pac_cover} we know that if we choose algorithm $\cA$ such that for every distribution over $[-1/2,1/2]^{p\times T}\times \{-1,1\}$ and any input $S$ of $m$ i.i.d. samples from $\cD$ it outputs $\cA(S) = \hat{h} = \arg\min_{h\in\cH_w}\frac{1}{|S|}\sum_{(x,y)\in S}l_{\gamma}(h,x,y)$ , then with probability at least $1-\delta$ we have
  \begin{equation}\label{eq:upper1}
\begin{aligned}
       &\expects{(U,y)\sim \cD}{l_{\gamma}(\hat{h},U,y)} \\
       &\leq\inf_{h\in\cH_{w}}\expects{(U,y)\sim \cD}{l_{\gamma}(h,U,y)} +  16\epsilon + \frac{24}{\sqrt{m}}\sqrt{\log N_U(\gamma\epsilon,\cH_w,m,\|.\|_2^{\ell_2})}+ 6 \sqrt{\frac{\log(2/\delta)}{2m}}.
\end{aligned}
\end{equation}
We know that $\rv{\cQ_w}$ is a class of functions from $[-1/2,1/2]^{p\times T}$ to $[-1/2,1/2]$. We also know that $\|x\|_2^{\ell_2}\leq\|x\|_2^{\infty}$ for every $x$. We can now use Theorem~\ref{thm:cover_derandomization} to turn the bound on the covering number of $\rv{\cQ_w}$ to a bound on the covering number of $\cE(\rv{\cQ_w})$. Note that Theorem~\ref{thm:cover_derandomization} is stated for functions with outputs in $[-B,B]$ and $\rv{\cQ_w}=\rv{\cG_{\sigma}}\circ\brmodel{\rvmnet{p_0}{p_k}{w}{\sigma}}{T}$ outputs values in $\rv{\cG_{\sigma,p_k}}\circ\rv{\cX_{0.5,p_k}}$. However, $\rv{\cG_{\sigma,p_k}}$ is a class of zero mean Gaussian random variables that are independent of the output of $\brmodel{\rvmnet{p_0}{p_k}{w}{\sigma}}{T}$ and, therefore, they do not change the expectation and the covering number bound for $\cE(\rv{\cQ_w})$ would be the same as the covering number bound for $\cE\left(\brmodel{\rvmnet{p_0}{p_k}{w}{\sigma}}{T}\right)$. Thus we know that 
\begin{equation*}
    N_U(\gamma\epsilon,\cH_w,m,\|.\|_2^{\ell_2}) \leq N_U(\gamma\epsilon,\cH_w,m,\|.\|_2^{\infty}) \leq  N_U(\gamma\epsilon,\rv{\cQ_w},\infty,d_{TV}^{\infty},\rv{\Delta_{0.5,p\times T}}).
\end{equation*}
We can, therefore, rewrite Equation~\ref{eq:upper1} as follows.
\begin{equation}\label{eq:upper2}
\begin{aligned}
        &\expects{(U,y)\sim \cD}{l_{\gamma}(\hat{h},U,y)} \\
       &\leq\inf_{h\in\cH_{w}}\expects{(U,y)\sim \cD}{l_{\gamma}(h,U,y)} +  16\epsilon + \frac{24}{\sqrt{m}}\sqrt{\log N_U(\gamma\epsilon,\rv{\cQ_2},\infty,d_{TV}^{\infty},\rv{\Delta_{0.5,p\times T}})}+ 6 \sqrt{\frac{\log(2/\delta)}{2m}}.
\end{aligned}
\end{equation}
Therefore, if we find $m$ such that $\frac{1}{\sqrt{m}}\sqrt{\log N_U(\gamma\epsilon,\rv{\cQ_w},\infty,d_{TV}^{\infty},\rv{\Delta_{p\times T}})} = O(\epsilon)$ and $\sqrt{\frac{\log(1/\delta)}{m}}=O(\epsilon)$ then we can guarantee $\expects{(U,y)\sim \cD}{l_{\gamma}(\hat{h},U,y)} \leq\inf_{h\in\cH_{w}} \expects{(U,y)\sim \cD}{l_{\gamma}(h,U,y)} +O(\epsilon)$.

We know of a covering number bound for $\rv{\cQ_w}$ from Theorem~\ref{thm:cover_noisy_rnn} which is as follows.

\begin{equation*}
        \begin{aligned}  &\log N_U\left(\epsilon,\rv{\cQ_w},\infty,d_{TV}^{\infty},\rv{\Delta_{0.5,p\times T}}\right)
    =O\left(w\log\left(\frac{wT}{\epsilon\sigma}\ln\left(\frac{wT}{\epsilon\sigma}\right)\right)\right).
        \end{aligned}
    \end{equation*}
    
We can thus write that
\begin{equation*}
    \begin{aligned}
     & \sqrt{\frac{\log N_U(\gamma\epsilon,\rv{\cQ_w},\infty,d_{TV}^{\infty},\rv{\Delta_{0.5,p\times T}})}{m}} = O(\epsilon)  \Leftrightarrow m = O\left(\frac{1}{\epsilon^2}w\log\left(\frac{wT}{\epsilon\sigma}\ln\left(\frac{wT}{\epsilon\sigma}\right)\right)\right)
    \end{aligned}
\end{equation*}
Moreover, if we want $\sqrt{\frac{\log(1/\delta)}{m}}=O(\epsilon)$ then we should have $m = O\left(\frac{\log(1/\delta)}{\epsilon^2}\right)$. Combining the above results, we can conclude that
\begin{equation*}
    m_{\cH_w}(\epsilon,\delta) = O\left(\frac{w\log\left(\frac{wT}{\epsilon\sigma}\ln\left(\frac{wT}{\epsilon\sigma}\right)\right) + \log(1/\delta)}{\epsilon^2}\right) =\widetilde{O}\left(\frac{w\log\left(\frac{T}{\sigma}\right)+\log(1/\delta)}{\epsilon^2}\right).
\end{equation*}
samples is sufficient to conclude that with probability at least $1-\delta$ we have $\allowdisplaybreaks \expects{(U,y)\sim \cD}{l_{\gamma}(\hat{h},U,y)} \leq\inf_{h\in\cH_{w}} \expects{(U,y)\sim \cD}{l_{\gamma}(h,U,y)}  + O(\epsilon)$, which implies PAC learning $\cH_w$ with respect to ramp loss with a sample complexity of $m_{\cH_w}(\epsilon,\delta)$.
\end{proof}

\section{PAC learning and covering number bounds}\label{app:pac_cover}
In this section, we discuss how we can find a bound on the sample complexity of PAC learning a class of functions with respect to ramp loss from a bound on its covering number. Particularly, we show how to use a bound on covering number to find the number of samples required to ensure uniform convergence with respect to ramp loss. We then connect the uniform convergence results to PAC learning and find the minimum number of samples required to guarantee PAC learning with respect to ramp loss.

We start by defining uniform convergence (with respect to ramp loss).
\begin{definition}[Uniform Convergence with Respect to Ramp Loss]
Let $\cF$ be a class of functions from $\cX$ to $\bR$. We say that $\cF$ has uniform convergence property with respect to ramp loss with margin parameter $\gamma>0$ if there exists some function $m:(0,1)^2\rightarrow \bN$ such that for every distribution $\cD$ over $\cX\times\{-1,1\}$ and every $\epsilon,\delta \in (0,1)$, if $S$ is a set of $m(\epsilon,\delta)$ i.i.d. samples from $\cD$, then with probability at least $1-\delta$ (over the randomness of $S$) for every function $f\in\cF$ we have $\left|\expects{(x,y)\sim\cD}{l_{\gamma}(f,x,y)} -\frac{1}{|S|}\sum_{(x,y)\in S}l_{\gamma}(f,x,y)\right|\leq \epsilon$.
\end{definition}
The \textit{sample complexity} of uniform convergence for class $\cF$ is denoted by $m_{\cF}^{\text{UC}}(\epsilon,\delta)$, which is the minimum number of samples required to guarantee uniform convergence for $\cF$.
We now show that uniform convergence implies PAC learning (with respect to ramp loss).
\begin{lemma}\label{lemma:uniconv_pac}
   Let $\cF$ be a class of functions from $\cX$ to $\bR$ that satisfies uniform convergence property with respect to ramp loss. Then for any $(\epsilon,\delta)\in (0,1)$, we have $m_{\cF}(\epsilon,\delta)\leq m_{\cF}^{\text{UC}}(\epsilon/2,\delta)$, i.e., there exists an algorithm $\cA$ such that for any distribution $\cD$ over $\cX\times\{-1,1\}$ and any $(\epsilon,\delta)\in(0,1)$, if $S$ is a set of $m \geq  m_{\cF}^{\text{UC}}(\epsilon/2,\delta)$ i.i.d. samples from $\cD$, then with probability at least $1-\delta$, we have that $\expect{(x,y)\sim \cD}{l_{\gamma}(\cA(S),x,y)}\leq \inf_{f\in\cF}\expects{(x,y)\sim \cD}{l_{\gamma}(f,x,y)}+\epsilon$.
\end{lemma}
\begin{proof}
    Let $\cA$ be an algorithm that outputs the function in $\cF$ that has the minimum empirical loss, i.e., $\cA(S) = \arg\min_{f\in\cF}\frac{1}{|S|}\sum_{(x,y)\in S}l_{\gamma}(f,x,y)$. Since $S$ is a set of $m\geq m_{\cF}^{\text{UC}}(\epsilon/2,\delta)$ samples, we know that with probability at least $1-\delta$ we have $\left|\expects{(x,y)\sim \cD}{l_{\gamma}(f,x,y)}-\frac{1}{|S|}\sum_{(x,y)\in S}l_{\gamma}(f,x,y)\right|\leq \epsilon/2$ for every $f\in\cF$. Let $\hat{f} = \cA(S)$. Then for every $f\in\cF$ we can write that
    \begin{equation*}
    \begin{aligned}
        &\expects{(x,y)\sim\cD}{l_{\gamma}(\hat{f},x,y)}\leq  \frac{1}{|S|}\sum_{(x,y)\in S}l_{\gamma}(\hat{f},x,y) + \frac{\epsilon}{2} \leq \frac{1}{|S|}\sum_{(x,y)\in S}l_{\gamma}(f,x,y) + \frac{\epsilon}{2} \\
        &\leq \expects{(x,y)\in \cD}{l_{\gamma}(f,x,y)} + \frac{\epsilon}{2}  + \frac{\epsilon}{2}  = \expects{(x,y)\in \cD}{l_{\gamma}(f,x,y)}+ \epsilon.
    \end{aligned}
    \end{equation*}
    This implies that with $m\geq m_{\cF}^{\text{UC}}(\epsilon/2,\delta)$ i.i.d. samples we can guarantee PAC learning with respect to ramp loss with parameters $\epsilon$ and $\delta$. In other words, we have $m_{\cF}(\epsilon,\delta)\leq m_{\cF}^{\text{UC}}(\epsilon/2,\delta)$.
\end{proof}
The following theorem tells us that we can relate the bound on the covering number of a class of functions to the uniform convergence property for that class. The proof relies on bounding the Rademacher complexity of the class by a bound on its covering number \citep{dudley2010universal} and then relating the bound on the Rademacher complexity to uniform convergence property. See \citet{shalev2014understanding} and \citet{mohri2018foundations} for a more detailed discussion and proof.
\begin{theorem}\label{thm:unif_cover}
Let $\cF$ be a class of functions from $\cX$ to $\bR$ and $\cF_{\gamma} = \{f_{\gamma}:\cX\times\{-1,1\}\to [0,1]\mid 
 f_{\gamma}(x,y) = r_{\gamma}\left(-f(x).y\right),f\in\cF\}$ be the class of its composition with ramp loss. Let $\cD$ be a distribution over $\cX\times \{-1,1\}$ and $S\sim\cD^m$ be an i.i.d. sample of size $m$. Then, for any $\delta\in(0,1)$ with probability at least $1-\delta$ (over the randomness of $S$) for every $f\in\cF$ we have
\begin{equation*}
\begin{aligned}
       &\expects{(x,y)\sim\cD}{l_{\gamma}(f,x,y)}\\
       &\leq \frac{1}{|S|}\sum_{(x,y)\in S}l_{\gamma}(f,x,y) +\inf_{\epsilon\in [0,1/2]}\left\{2\left[4\epsilon + \frac{12}{\sqrt{m}}\int_{\epsilon}^{1/2}\sqrt{\log N_U(\nu,\cF_{\gamma},m,\|.\|_2^{\ell_2})}\,d\nu\right]\right\}+ 3 \sqrt{\frac{\log(2/\delta)}{2m}}.
\end{aligned}
\end{equation*}
\end{theorem}
It is only left to find a bound on the covering number of $\cF_{\gamma}$ from a bound on the covering number of $\cF$. The following lemma helps us finding this bound.
\begin{lemma}[From Covering Number of $\cF$ to Covering Number of $\cF_{\gamma}$]\label{lemma:cover_to_margin}
Let $\cF$ be a class of functions from $\cX$ to $\bR$ and $\cF_{\gamma} = \{f_{\gamma}:\cX\times\{-1,1\}\to [0,1]\mid 
 f_{\gamma}(x,y) = r_{\gamma}\left(-f(x).y\right),f\in\cF\}$ be the class of its composition with ramp loss. Then we have
\begin{equation*}
    N_U(\epsilon,\cF_{\gamma},m,\|.\|_2^{\ell_2}) \leq N_U(\gamma\epsilon,\cF,m,\|.\|_2^{\ell_2}). 
\end{equation*}
\end{lemma}
\begin{proof}
First, it is easy to verify that $r_{\gamma}$ (with respect to the first input) is a Lipschitz continuous function with respect to $\|.\|_2$ with Lipschitz factors of $1/\gamma$; see e.g., section A.2 in \citet{bartlett2017spectrally}. 

Fix an input set $S=\{(x_1,y_1),\ldots,(x_m,y_m)\}\subset \cX \times \cY$ and let $C=\{\hat{f_i}_{|S}\mid \hat{f_i}\in\cF,\, i\in[r]\}$ be an $(\gamma\epsilon)$-cover for $\cF_{|S}$. For the simplicity of notation, we denote the composition of $\hat{f_i}$ with ramp loss by $\hat{f}_{{\gamma,i}}$. Now, we prove that $C_{\gamma}=\{\hat{f}_{\gamma,i_{|S}}\mid \hat{f}_{{\gamma,i}}\in\cF_{\gamma},\, i\in[r]\}$ is also an $\epsilon$-cover for ${\cF_\gamma}_{|S}$.

Given any $f\in\cF$, there exists $\hat{f_i}_{|S}\in C$ such that
\begin{equation*}
    \left\| (\hat{f}_i(x_1),\ldots,\hat{f}_i(x_m))-(f(x_1),\ldots,f(x_m))\right\|_2^{\ell_2}\leq \gamma\epsilon.
\end{equation*}
We can then write that
\begin{equation}\label{eq:B1}
    \begin{aligned}
          &\left\| (\hat{f}_{{\gamma,i}}(x_1),\ldots,\hat{f}_{{\gamma,i}}(x_m))-(f_{\gamma}(x_1),\ldots,f_{\gamma}(x_m))\right\|_2^{\ell_2}\\
      & = \sqrt{\frac{1}{m} \sum_{k=1}^{m} \left( \hat{f}_{{\gamma,i}}(x_k) - f_{\gamma}(x_k) \right)^2}\\
      & =  \sqrt{\frac{1}{m} \sum_{k=1}^{m} \left( r_{\gamma}\left(-\hat{f}_i(x_k).y_k\right) - r_{\gamma}(-f(x_k).y_k)\right)^2}.
    \end{aligned}
\end{equation}
From the Lipschitz continuity of $r_{\gamma}\left(x\right)$ we can conclude that for any $(x,y)\in\cX \times \cY$,
\begin{equation*}
    \begin{aligned}
     \left|r_{\gamma}\left(-f(x).y\right) - r_{\gamma}(-\hat{f}_i(x).y)\right|\leq \frac{1}{\gamma} \left|\hat{f}_i(x) - f(x)\right|
    \end{aligned}.
\end{equation*}
Taking the above equation into account, we can rewrite Equation~\ref{eq:B1} as
\begin{equation*}
    \begin{aligned}
      &\left\| (\hat{f}_{{\gamma,i}}(x_1),\ldots,\hat{f}_{{\gamma,i}}(x_m))-(f_{\gamma}(x_1),\ldots,f_{\gamma}(x_m))\right\|_2^{\ell_2}\\
      &\leq \frac{1}{\gamma} \sqrt{\frac{1}{m} \sum_{k=1}^{m} \left( (\hat{f}_i(x_k) - f(x_k))\right)^2}\\
      & \leq   \frac{1}{\gamma} \left\| (\hat{f}_i(x_1),\ldots,\hat{f}_i(x_m))-(f(x_1),\ldots,f(x_m))\right\|_2^{\ell_2}\\
      &\leq \frac{1}{\gamma}\gamma\epsilon\\
      &\leq \epsilon.
    \end{aligned}
\end{equation*}
In other words, for any $f_{\gamma_{|S}}\in{\cF_\gamma}_{|S}$ there exists $\hat{f}_{\gamma,i_{|S}}\in S$ such that $\left \|\hat{f}_{\gamma,i_{|S}} - f_{\gamma_{|S}}\right\|_2^{\ell_2}\leq \epsilon$ and, therefore, $C_{\gamma}$ is an $\epsilon$-cover for ${\cF_\gamma}_{|S}$ and the result follows.
\end{proof}

We can now combine Theorem~\ref{thm:unif_cover}, Lemma~\ref{lemma:uniconv_pac}, and Lemma~\ref{lemma:cover_to_margin} to state the following theorem, which implies that we can relate a bound on the covering number of a class $\cF$ to PAC learning $\cF$ with respect to ramp loss.

\begin{theorem}\label{thm:pac_cover_exact}
Let $\cF$ be a class of functions from $\cX$ to $\bR$. There exists an algorithm $\cA$ with the following property: For every distribution $\cD$ over $\cX\times \{-1,1\}$ and every $\delta \in (0,1)$, if $S$ is a set of $m$ i.i.d. samples from $\cD$, the algorithm outputs a hypothesis $f = \cA(S)$ such that with probability at least $1-\delta$ (over the randomness of $S$ and $\cA$) we have
\begin{equation*}
\begin{aligned}
       &\expects{(x,y)\sim \cD}{l_{\gamma}(f,x,y)}\\
       &\leq \inf_{f\in\cF}\expects{(x,y)\sim \cD}{l_{\gamma}(f,x,y)}+2\inf_{\epsilon\in [0,1/2]}\left\{2\left[4\epsilon + \frac{12}{\sqrt{m}}\int_{\epsilon}^{1/2}\sqrt{\log N_U(\nu,\cF_{\gamma},m,\|.\|_2^{\ell_2})}\,d\nu\right]\right\}+ 6 \sqrt{\frac{\log(2/\delta)}{2m}}.
\end{aligned}
\end{equation*}
\end{theorem}
\begin{proof}
    From Theorem~\ref{thm:unif_cover} we know that for every $f\in\cF$ with probability at least $1-\delta$ we have
    \begin{equation*}
\begin{aligned}
       &\expects{(x,y)\sim \cD}{l_{\gamma}(f,x,y)} \\
       & \leq \frac{1}{|S|}\sum_{(x,y)\in S}{l_{\gamma}(f,x,y)} +\inf_{\epsilon\in [0,1/2]}\left\{2\left[4\epsilon + \frac{12}{\sqrt{m}}\int_{\epsilon}^{1/2}\sqrt{\log N_U(\nu,\cF_{\gamma},m,\|.\|_2^{\ell_2})}\,d\nu\right]\right\}+ 3 \sqrt{\frac{\log(2/\delta)}{2m}}.
\end{aligned}
\end{equation*}
Lemma~\ref{lemma:uniconv_pac} suggests that if we choose algorithm $\cA$ such that $\cA(S) = \hat{f} = \arg\min_{f\in\cF}\frac{1}{|S|}\sum_{(x,y)\in S}l_{\gamma}(f,x,y)$ then for any $f\in\cF$ with probability at least $1-\delta$ we have 
    \begin{equation*}
\begin{aligned}
       &\expects{(x,y)\sim \cD}{l_{\gamma}(\hat{f},x,y)} \\
       &\leq \frac{1}{|S|}\sum_{(x,y)\in S}l_{\gamma}(\hat{f},x,y) +\inf_{\epsilon\in [0,1/2]}\left\{2\left[4\epsilon + \frac{12}{\sqrt{m}}\int_{\epsilon}^{1/2}\sqrt{\log N_U(\nu,\cF_{\gamma},m,\|.\|_2^{\ell_2})}\,d\nu\right]\right\}+ 3 \sqrt{\frac{\log(2/\delta)}{2m}}\\
       &\leq \frac{1}{|S|}\sum_{(x,y)\in S}l_{\gamma}(f,x,y) +\inf_{\epsilon\in [0,1/2]}\left\{2\left[4\epsilon + \frac{12}{\sqrt{m}}\int_{\epsilon}^{1/2}\sqrt{\log N_U(\gamma\nu,\cF,m,\|.\|_2^{\ell_2})}\,d\nu\right]\right\}+ 3 \sqrt{\frac{\log(2/\delta)}{2m}}\\
       &\leq \expects{(x,y)\sim \cD}{l_{\gamma}(f,x,y)} +2\inf_{\epsilon\in [0,1/2]}\left\{2\left[4\epsilon + \frac{12}{\sqrt{m}}\int_{\epsilon}^{1/2}\sqrt{\log N_U(\gamma\nu,\cF,m,\|.\|_2^{\ell_2})}\,d\nu\right]\right\}+ 6 \sqrt{\frac{\log(2/\delta)}{2m}}\\
        &\leq  \inf_{f\in\cF} \expects{(x,y)\sim \cD}{l_{\gamma}(f,x,y)} +2\inf_{\epsilon\in [0,1/2]}\left\{2\left[4\epsilon + \frac{12}{\sqrt{m}}\int_{\epsilon}^{1/2}\sqrt{\log N_U(\gamma\nu,\cF,m,\|.\|_2^{\ell_2})}\,d\nu\right]\right\}+ 6 \sqrt{\frac{\log(2/\delta)}{2m}}.
\end{aligned}
\end{equation*}
\end{proof}
In Appendix~\ref{app:upper_bound} we use the above theorem together with an approximation of the right hand side of the above inequality to find an upper bound on the sample complexity of PAC learning noisy recurrent neural networks with respect to ramp loss.

\section{Missing proof from Section~\ref{sec:covering_random}}\label{app:covering}
\subsection{Proof of Theorem~\ref{thm:cover_derandomization}}
\begin{proof}
Let $S=\{U_1,\ldots,U_m\}\subset\bR^{p\times T}$ be an input set and define $\rv{S}=\{\rv{U_1},\ldots,\rv{U_m}\}\subset \rv{\Delta_{p\times T}}$. Let $C=\{\rv{\hat{{f_1}}}_{|\rv{S}},\ldots,\rv{\hat{{f_r}}}_{|\rv{S}} \mid \rv{\hat{f}_r} \in \rv{\cF}, i \in [r]\}$ be an $\epsilon$-cover for $\rv{\cF}_{|\rv{S}}$ with respect to $d_{TV}^{\infty}$. Denote $\cH = \cE(\rv{\cF})$ and let $\hat{\cH}=\left\{\hat{h}_i(x)=\expects{\rv{\hat{f_i}}}{~\rv{\hat{f_i}}({x})} \mid i \in [r]\right\}\subset \cE(\rv{\cF})$ be a new set of non-random function.

Given any random function $\rv{f}\in\rv{\cF}$ and considering the fact that $C$ is an $\epsilon$-cover for $\rv{\cF}_{|\rv{S}}$ we know there exists $\rv{\hat{f_i}},\,i\in[r
]$ such that 
\begin{equation*}
    d_{TV}^{\infty}\left(\rv{\hat{f_i}}_{|\rv{S}},\rv{f}_{|\rv{S}}\right)=d_{TV}^{\infty}\left((\rv{\hat{f_i}}(\rv{U_1}),\ldots,\rv{\hat{f_i}}(\rv{U_m})),(\rv{f}(\rv{U_1}),\ldots,\rv{f}(\rv{U_m}))\right)\leq \epsilon.
\end{equation*}
From the above equation we can conclude that for any $k \in [m]$ we have $d_{TV}\left(\rv{\hat{f_i}}(\rv{U_k}),\rv{f}(\rv{U_k})\right)\leq \epsilon$. Further, for the corresponding $ h,\hat{h}_i \in \cE(\rv{\cF})$, we know that
\begin{equation*}
    \begin{aligned}
           &\hat{h}_i(U_k) = \expects{\rv{\hat{f_i}}}{~\rv{\hat{f_i}}({\rv{U_k}})} = \int_{\bR^d}  x\mD(\rv{\hat{f_i}}({\rv{U_k}}))(x)dx,\\
            &h(U_k) = \expects{\rv{f}}{~\rv{f}({\rv{U_k}})}=\int_{\bR^d}  x\mD(\rv{f}({\rv{U_k}}))(x)dx.
    \end{aligned}
\end{equation*}
 Denote $I = \mD(\rv{f}({\rv{U_k}}))$ and $\hat{I} = \mD(\rv{\hat{f_i}}({\rv{U_k}}))$. Define two new density functions $I_{diff}$ and $\hat{I}_{diff}$ as
\begin{equation*}
    \begin{aligned}
      I_{diff}(x) &= \left\{
    \begin{array}{ll}
       \displaystyle \frac{I(x) - \hat{I}(x)}{d_{TV}(I,\hat{I})} &  I(x)\geq \hat{I}(x)\\\\
        0 & \text{otherwise,} 
    \end{array}
    \right. \\
     \hat{I}_{diff}(x) &= \left\{
    \begin{array}{ll}
       \displaystyle \frac{\hat{I}(x) - I(x)}{d_{TV}(I,\hat{I})} &  \hat{I}(x)\geq I(x)\\\\
        0 & \text{otherwise.} 
    \end{array}
    \right.
    \end{aligned}
\end{equation*}
Also, we define $I_{min}$ as 
\begin{equation*}
    I_{min}(x) = \frac{\min\{I(x),\hat{I}(x)\}}{\int\min\{I(x),\hat{I}(x)\}dx} = \frac{\min\{I(x),\hat{I}(x)\}}{1-d_{TV}(I,\hat{I})}.
\end{equation*}
We can verify that
\begin{equation*}
    \begin{aligned}
     I(x) &= \left(1-d_{TV}(I,\hat{I})\right)I_{min}(x) + d_{TV}(I,\hat{I}).I_{diff}(x)\\
     \hat{I}(x) &= \left(1-d_{TV}(I,\hat{I})\right)I_{min}(x) + d_{TV}(I,\hat{I}).\hat{I}_{diff}(x).
    \end{aligned}
\end{equation*}
We then find the $\ell_2$ distance between $\hat{h}_i(U_k)$ and $h(U_k)$ by
\begin{equation*}
    \begin{aligned}
           &\left\|\hat{h}_i(U_k)-h(U_k)\right\|_2\\
           &=\left\|\int_{\bR^d}  x\hat{I}(x)dx - \int_{\bR^d}  xI(x)dx \right\|_2\\
           &=\left\|\int_{\bR^d}  x\left[\left(1-d_{TV}(I,\hat{I})\right)I_{min}(x) + d_{TV}(I,\hat{I}).\hat{I}_{diff}(x)\right]\right.\\ 
           & \left. - x\left[\left(1-d_{TV}(I,\hat{I})\right)I_{min}(x) + d_{TV}(I,\hat{I}).I_{diff}(x)\right]dx \right\|_2\\
           & = \left\|\int_{\bR^d}  xd_{TV}(I,\hat{I})\left[\hat{I}_{diff}(x) - I_{diff}(x)\right]dx\right\|_2\\
           & = d_{TV}(I,\hat{I})\left\|\int_{\bR^d}x\left[\hat{I}_{diff}(x) - I_{diff}(x)\right]dx\right\|_2 
           \\
           &\leq 2B\sqrt{q}\,d_{TV}\left(\rv{f}(\rv{U_k}),\rv{\hat{f_i}}(\rv{U_k})\right) && \text{(Bounded domain $[-B,B]^q$ and triangle inequality)}\\
           &\leq 2B\epsilon\sqrt{q}.
    \end{aligned}
\end{equation*}
Since this result holds for any $k \in [m]$, we have 
\begin{equation*}
    \begin{aligned}
     \|\hat{h_i}_{|S}-h_{|S}\|_2^{\ell_2} &= \sqrt{\frac{1}{m}\sum_{k=1}^m\left\|\hat{h_i}(U_k)-h(U_k)\right\|_2^2}\\
    & \leq \sqrt{\frac{1}{m}\sum_{k=1}^m (2B\sqrt{q})^2\left(d_{TV}\left(\rv{f}(\rv{U_k}),\rv{\hat{f_i}}(\rv{U_k})\right)\right)^2} \leq 2B\sqrt{q} \sqrt{\frac{1}{m}\sum_{k=1}^m \epsilon^2}\leq 2B\epsilon\sqrt{q}.
    \end{aligned}
\end{equation*}

Therefore, $\hat{\cH}_{|S}$ is a $2B\epsilon\sqrt{q}$ cover for $\cH_{|S}$ with respect to $\|.\|_2^{\ell_2}$ and $|\hat{\cH}_{|S}|=r$. This holds for any subset $S$ of $\bR^{p \times T}$ with $|S|=m$. Therefore,
\begin{equation*}
     N_U(2B\epsilon\sqrt{q},\cE(\rv{\cF}),m,\|.\|_2^{\ell_2}) \leq N_U(\epsilon,\rv{\cF},m,d_{TV}^{\infty},\rv{\Delta_{p\times T}})\leq N_U(\epsilon,\rv{\cF},\infty,d_{TV}^{\infty},\rv{\Delta_{p\times T}}).
\end{equation*}
\end{proof}
\end{document}